\documentclass[10pt,journal,compsoc,table]{IEEEtran}

\usepackage{xspace}
\usepackage{adjustbox}
\usepackage{booktabs}
\usepackage{multirow}
\usepackage{makecell}
\usepackage{footnote}
\usepackage{amsfonts}
\usepackage{bm}
\usepackage{amsmath}
\usepackage{amsthm}
\usepackage{amssymb} 
\usepackage[linesnumbered,ruled,vlined]{algorithm2e}
\usepackage{url}
\usepackage{xcolor}
\usepackage{subcaption}
\usepackage{tablefootnote}
\usepackage{dsfont}
\usepackage{enumitem}

\graphicspath{{figs/}}

\newcommand{\secref}[1]{\mbox{Section~\ref{#1}}}
\newcommand{\thref}[1]{\mbox{Theorem~\ref{#1}}}
\newcommand{\claimref}[1]{\mbox{Claim~\ref{#1}}}

\newcommand{\lemref}[1]{\mbox{Lemma~\ref{#1}}}

\renewcommand{\eqref}[1]{\mbox{Eq.~(\ref{#1})}}
\newcommand{\tabref}[1]{\mbox{Table~\ref{#1}}}
\newcommand{\figref}[1]{\mbox{Fig.~\ref{#1}}}
\newcommand{\algoref}[1]{\mbox{Algorithm~\ref{#1}}}

\newcommand{\etal}{\textit{et al}.}
\newcommand{\ie}{\textit{i}.\textit{e}.,}
\newcommand{\eg}{\textit{e}.\textit{g}.,}

\def\R{\mathbb{R}}

\def\Dist#1{{\mathcal{#1}}}

\newtheorem{claim}{Claim}
\newtheorem{defn}{Definition}
\newtheorem{thm}{Theorem}
\newtheorem{lemma}{Lemma}
\newtheorem{remark}{Remark}

\newtheorem{prop}{Proposition}
\newtheorem{corrolary}{Corollary}
\newenvironment{proof-sketch}{%
  \renewcommand{\proofname}{Proof Sketch}\proof}{\endproof}

\usepackage{tikz,xcolor,hyperref}
\definecolor{lime}{HTML}{A6CE39}
\DeclareRobustCommand{\orcidicon}{%
  \begin{tikzpicture}
  \draw[lime, fill=lime] (0,0) 
  circle [radius=0.16] 
  node[white] {{\fontfamily{qag}\selectfont \tiny ID}}; \draw[white, fill=white] (-0.0625,0.095) 
  circle [radius=0.007];  \end{tikzpicture}
  \hspace{-2mm}}
\foreach \x in {A, ..., Z}{%
  \expandafter\xdef\csname orcid\x\endcsname{\noexpand\href{https://orcid.org/\csname orcidauthor\x\endcsname}{\noexpand\orcidicon}}
  }

\SetCommentSty{mycommfont}

\makeatletter
\renewenvironment{proof}[1][\relax]{\par
  \pushQED{\qed}%
  \normalfont \topsep6\p@\@plus6\p@\relax
  \trivlist
  \item[\hskip\labelsep\itshape
    \ifx#1\relax \proofname\else\proofname{} of #1\fi\@addpunct{.}]\ignorespaces
}{%
  \popQED\endtrivlist\@endpefalse
}
\makeatother

\ifCLASSOPTIONcompsoc
  \usepackage[nocompress]{cite}
\else
  \usepackage{cite}
\fi
\ifCLASSINFOpdf
\else
\fi
\begin{document}
%
\title{Fair Representation: Guaranteeing Approximate Multiple Group Fairness for Unknown Tasks}
%
%
%
%

\author{Xudong~Shen\orcidA{},
        Yongkang~Wong\orcidB{},~\IEEEmembership{Member,~IEEE,}
        and~Mohan~Kankanhalli\orcidC{},~\IEEEmembership{Fellow,~IEEE}
\IEEEcompsocitemizethanks{
\IEEEcompsocthanksitem X. Shen is with the Graduate School, National University of Singapore, Singapore, 117456. E-mail: xudong.shen@u.nus.edu
\IEEEcompsocthanksitem Y. Wong and M. Kankanhalli are with the School of Computing, National University of Singapore, Singapore, 117417. E-mail: yongkang.wong@nus.edu.sg, mohan@comp.nus.edu.sg.}
\thanks{Manuscript received September XX, 2021; revised September XX, 2021.}}

%
%

\markboth{IEEE TRANSACTIONS ON PATTERN ANALYSIS AND MACHINE INTELLIGENCE,~Vol.~XX, No.~X, September~2021}%
{Shen \MakeLowercase{\textit{et al.}}: Fair Representation}
%



\IEEEtitleabstractindextext{%
\begin{abstract}
Motivated by scenarios where data is used for diverse prediction tasks, 
we study whether fair representation can be used to guarantee fairness \emph{for unknown tasks} and \emph{for multiple fairness notions}.
We consider seven group fairness notions that cover the concepts of independence, separation, and calibration.
Against the backdrop of the fairness impossibility results, we explore approximate fairness.
We prove that, although fair representation might not guarantee fairness for all prediction tasks, it does guarantee fairness for an important subset of tasks---the tasks for which the representation is discriminative.
Specifically, all seven group fairness notions are linearly controlled by fairness and discriminativeness of the representation.
When an incompatibility exists between different fairness notions, 
fair and discriminative representation hits the sweet spot that \emph{approximately} satisfies all notions.
Motivated by our theoretical findings, we propose to learn both fair and discriminative representations using pretext loss which self-supervises learning, and Maximum Mean Discrepancy as a fair regularizer.
Experiments on tabular, image, and face datasets show that using the learned representation, downstream predictions \emph{that we are unaware of when learning the representation} indeed become fairer.
The fairness guarantees computed from our theoretical results are all valid.
\end{abstract}

\begin{IEEEkeywords}
Fair Representation, Group Fairness, Fair Machine Learning
\end{IEEEkeywords}}

\maketitle

\IEEEdisplaynontitleabstractindextext

%
\IEEEpeerreviewmaketitle


%
%
%
%

 

\IEEEraisesectionheading{
\section{Introduction}
\label{sec:introduction}
}

\IEEEPARstart{A}{s} machine learning (ML) continues to be widely used in various social contexts, \textit{fairness} grows to be a prominent concern.
There is accumulating evidence that ML systems institutionalize discrimination, for example in healthcare~\cite{obermeyer_Science_2019}, face recognition~\cite{buolamwini_FAccT_2018_gender}, and criminal justice~\cite{angwin_2016_propublica}.

In many scenarios, the collected dataset is shared as-is with dataset users who then perform diverse prediction tasks.
Although the dataset owner has fairness concerns, they might have no knowledge of what tasks are performed using the dataset and the dataset users might bias certain demographics as long as it maximizes their utility.
For example, digital footprint data is predictively useful for various target attributes, such as loan default~\cite{berg2018rise}, sexual orientation~\cite{jernigan2009gaydar}, and political ideology~\cite{kosinski2013private,preoctiuc2017beyond}.
When the data owner sells a digital footprint dataset on data marketplaces, they\footnote{We use the singular gender-neutral pronoun ``they'' whenever possible.} neither knows what prediction tasks will be performed nor trusts the dataset users.
This results in potential fairness violations.

A popular approach to achieve fair outcome is through fair representation.
The intuition is, by first mapping input data to an intermediate representation that satisfies certain conditions,
\eg~anonymizing the sensitive attribute,
downstream predictions based on the representation are guaranteed to be fair.
Besides the works~\cite{beutel_arxiv_2017,edwards_arxiv_2015,xie_NeurIPS_2017,zhao_ICLR_2020_conditional,kingma_arxiv_2013_VAE,moyer_NeurIPS_2018,song_AISTATS_2019,louizos_arxiv_2015,creager_ICML_2019,kehrenberg2020null,quadrianto2019discovering} that explicitly take this approach, a large number of works in fact implicitly create fair representations, for example in computer vision by collecting balanced datasets~\cite{buolamwini_FAccT_2018_gender,karkkainen2021fairface} 
or generating synthetic datasets~\cite{choi2020fair,ramaswamy2021fair}; and in natural language processing by de-biasing word/sentence embeddings~\cite{bolukbasi_NeurIPS_2016_word_embeddings,caliskan2017semantics,zhao2018learning,may2019measuring} or using contrast sets to augment training corpora~\cite{gardner2020evaluating}.

However, despite the popularity, the understanding of which fairness notion can be guaranteed to what degree is limited.
It is known that downstream predictions' Statistical Parity is upper bounded by the sensitive attribute's predictability~\cite{feldman_SIGKDD_2015,mcnamara_arxiv_2017}.
But for other fairness notions such as Disparity of Opportunity, existing work~\cite{madras_ICML_2018} has required the data owner to (1) know what are the downstream prediction tasks,
(2) have access to the target labels, and (3) pose different constraints when learning the representation.
Consequently, the learned fair representation is tailored for the known prediction task and the specific fairness notion.

Thus, this work is motivated to study whether the data owner can learn a fair representation that guarantees various potentially unknown prediction tasks are fair.
Furthermore, we consider seven different group fairness notions as defined in~\tabref{tab:fairness_definitions} and \emph{aim to (approximately) satisfy them simultaneously}.
Statistical Parity (SP), Disparity of Opportunity (DOpp), Disparity of Regret (DR), and Disparity of Odds (DOdds) are commonly used fairness notions~\cite{calders_DMKD_2010,hardt_NeurIPS_2016_equality,narayanan2018translation}.
Disparity of Calibration~(DC), Disparity of Positive Calibration (DPC), and Disparity of Negative Calibration (DNC) are proposed in this work as measures of uncalibratedness.
These notions cover all three aspects of group fairness, namely independence, separation, and calibration.
See \secref{sec:group_fairness} for a detailed discussion.

Satisfying multiple fairness notions simultaneously is a natural desideratum.
Fairness is a complex notion and any concrete definition can only capture a limited view.
Although the impossibility results~\cite{chouldechova2017fair,kleinberg2018inherent,barocas-hardt-narayanan} state that achieving different group fairness notions exactly is impossible when the sensitive and target attributes are coupled,
from a practical perspective, we are still very interested in achieving them approximately.
%
Our considered problem encompasses and generalizes existing fair representation literature~\cite{edwards_arxiv_2015,louizos_arxiv_2015,beutel_arxiv_2017,feldman_SIGKDD_2015,mcnamara_arxiv_2017,xie_NeurIPS_2017,madras_ICML_2018,moyer_NeurIPS_2018,song_AISTATS_2019,creager_ICML_2019,zhao_ICLR_2020_conditional,oneto_2020_MMD},
which study different fairness notions in isolation.
They also have not considered whether fair representation can achieve calibration, \ie~DPC, DNC, and DC.

A common misunderstanding is, using a fair representation that does not contain any sensitive information, the predictions will be fair for \emph{all} fairness definitions, rendering our research question trivial.
The fact is, using a perfectly fair representation, the downstream predictions can still be maximally unfair w.r.t. DOpp, DR, and DOdds.
The intuition is, although fair representation ensures the sensitive attribute is unpredictable \emph{at the population level}, it might still be exploited \emph{in the subgroups defined by the target attribute} to produce unfair outcomes.


The contribution of this work is a better understanding of approximate fair representation.
We define representation fairness as the unpredictability of sensitive attribute and leverage another discriminativeness assumption that says the target attribute the user tries to predict (which equivalently defines a prediction task) is predictable from the representation.
We show that, although fair representation might not guarantee fairness for all prediction tasks, it does guarantee (approximately) seven group fairness notions for an important subset of tasks---\emph{the tasks for which the representation is discriminative}.
Thus, discriminativeness establishes a dichotomy between the prediction tasks whose fairness can be guaranteed, and those that cannot.
By encouraging fair representation to summarize more information about the data, better fairness guarantees can be obtained for a larger class of prediction tasks.

In light of the fairness impossibility results~\cite{chouldechova2017fair,kleinberg2018inherent,barocas-hardt-narayanan},
fair and discriminative representation hits the sweet spot that \emph{approximately} satisfies seven group fairness notions.
It provides a concrete way to navigate through the inherent incompatibility and answers affirmatively to the question raised in~\cite{barocas-hardt-narayanan}, which asks whether meaningful trade-off can be achieved when different group fairness notions become incompatible.

Our key contributions are summarized as follows.
\begin{itemize}
    \item 
    We consider seven group fairness notions that cover the concepts of independence, separation, and calibration, shown in~\tabref{tab:fairness_definitions}.
    We prove a sharp characterization that (1) perfectly fair and perfectly discriminative representation guarantees them  exactly and 
    (2) approximately fair and approximately discriminative representation guarantees them approximately.
    All seven group fairness notions are \emph{linearly} controlled by fairness and discriminativeness of the representation. 
    Our theoretical results are summarized in \tabref{tab:fairness_bounds}.
    \item 
    When the target and sensitive attributes are coupled, we show that a trade-off arises between representation fairness and discriminativeness, and thus a perfectly fair and perfectly discriminative representation is impossible.
    Nonetheless, an approximately fair and approximately discriminative representation can still be used to approximately satisfy seven group fairness notions.
    \item 
    Motivated by our theoretical findings, when data is subject to unknown downstream tasks, we propose to learn both fair and discriminative representations using pretext loss, which self-supervises the representation to summarize important semantics, and Maximum Mean Discrepancy~\cite{gretton_JMLR_2012_mmd}, which is used as a fair regularization. We cast it as a constrained optimization problem and solve it using a two-player game formulation.
    \item 
    Experiments on tabular and image datasets verify that 
    (1) using the learned representation, downstream prediction tasks \emph{that we are unaware of when learning the representation} indeed become fairer for seven group fairness notions, and (2) the fairness guarantees computed from our theoretical results are all valid.
    We also demonstrate a real-world application on learning gender-blind face representations, using which various facial attributes' predictions become fairer.\footnote{All code associated with this work can be found at \url{https://github.com/XudongOliverShen/2021-fair-representation}.}
\end{itemize}

The rest of the paper is structured as follows.
\secref{sec:related_work} reviews related work.
\secref{sec:preliminary} introduces the problem and necessary notions.
\secref{sec:perfect_fairness} and \ref{sec:approximate_fairness} elaborate our main theoretical results. 
\secref{sec:discussion} presents discussions.
\secref{sec:learning_general_fair_representation} delineates the proposed fair representation learning approach and \secref{sec:experiment} reports the results.
\secref{sec:conclusion} is the conclusion.

\section{Related Work}
\label{sec:related_work}

\noindent\textbf{Theoretical Understanding of Fair Representation~}
Existing works~\cite{feldman_SIGKDD_2015,mcnamara_arxiv_2017,madras_ICML_2018,zhao_NeurIPS_2019,oneto_2020_MMD} have shown that representations can satisfy different constraints to achieve different fairness notions \emph{individually}.
Feldman~\etal~\cite{feldman_SIGKDD_2015} shows that disparate impact can be bounded by the Bayes-optimal Balanced Error Rate (BER).
McNamara~\etal~\cite{mcnamara_arxiv_2017} and Zhao and Gordon~\cite{zhao_NeurIPS_2019} show that statistical parity can also be bounded.
Madras~\etal~\cite{madras_ICML_2018} proposes adversarial bounds on the Disparity of Opportunity, Regret, and Odds, but requires using the target labels.
Existing works study different fairness notions in isolation and assume knowing the downstream prediction task.

\vspace{1ex}
\noindent\textbf{Learning Fair Representation~}
Existing fair representation learning methods can be categorized into Generative Adversarial Network (GAN) based methods and Variational Autoencoder (VAE) based methods.
GAN-based methods use adversarial learning: an encoder learns discriminative representations and a critic tries to infer the sensitive attribute from the representation.
The critic can be configured to directly predict the sensitive attribute~\cite{beutel_arxiv_2017,edwards_arxiv_2015,madras_ICML_2018,xie_NeurIPS_2017}.
Kim and Cho~\cite{kim_AAAI_2020} designs an information-theoretic critic. Zhao~\etal~\cite{zhao_ICLR_2020_conditional} posits two critics, one for each target group.
VAE-based methods enrich the objective of VAE~\cite{kingma_arxiv_2013_VAE} with regularizers.
Moyer~\etal~\cite{moyer_NeurIPS_2018} and Song~\etal~\cite{song_AISTATS_2019} additionally minimize variational upper bound(s) of mutual information.
Louizos~\etal~\cite{louizos_arxiv_2015} adds a fastMMD~\cite{zhao_NC_2015_fastmmd} regularizer.
Creager~\etal~\cite{creager_ICML_2019} additionally enforces disentanglement to learn flexible fair representations that are fair w.r.t. multiple sensitive attributes.
%
Existing works assume knowing what prediction task the representation will be used for and thus use the target labels when learning fair representation.

\vspace{1ex}
\noindent\textbf{Fairness Impossibilities and Trade-offs~}
Impossibility results~\cite{chouldechova2017fair,kleinberg2018inherent,barocas-hardt-narayanan} have shown that achieving multiple fairness notions exactly is impossible when the target and sensitive attributes are coupled.
Nonetheless, in practice we are still very interested in approximately satisfying multiple fairness notions.
This work complements the impossibility results and studies how we can navigate through this incompatibility.
Trade-offs have also been shown between fairness and utility~\cite{menon_FAcct_2018,zhao_NeurIPS_2019,dutta2020there}, but for specific fairness notions and/or in the general prediction setting.
We provide a characterization of the trade-off specific to fair representation and encompasses all seven group fairness notions.
A concurrent work~\cite{lechner2021impossibility} studies impossibility results for fair representation.
They show that using fair representation, SP and DOdds can still be large if the data distribution can shift arbitrarily.
Our work complements their results and show what can be achieved without distribution shift.
\section{Preliminary}
\label{sec:preliminary}

\subsection{Problem Formulation}
\noindent\textbf{Notation}
Let $X \in \R^{|X|}$, $S \in \{0,1\}$, $Y \in \{0,1\}$ denote the raw data, the sensitive attribute, and the target attribute, respectively.
We use $f:\R^{|X|} \rightarrow \R^{|Z|}$ to denote the (data owner's) encoding function that maps the raw data point $X$ to its representation $Z$.
A dataset user is a randomized predictor $h:\R^{|Z|} \rightarrow [0,1]$, with $h(Z)$ denotes the probability of positive prediction.
The user's deterministic prediction $\hat{Y}$ is sampled from a Bernoulli distribution $Bern(h(Z))$.

For ease of exposition, we define some population base rates as
\begin{equation}\label{eq:base_rates}
\begin{cases}
    r \doteq P(S=1) ,\\
    a \doteq P(Y=1 \mid S=0) ,\\
    b \doteq P(Y=1 \mid S=1) ,
\end{cases}
\end{equation}
where $r$ is the relative sensitive group size, and $a$ and $b$ are the positive rates among different sensitive groups.
Without loss of generality, we assume $r,a,b \in (0,1)$.

We use curly letters to denote distributions, for example $\Dist{Z}$ is the distribution of $Z$.
We additionally denote the distributions of $Z$ in different subpopulations as 
\begin{equation}\label{eq:conditional_Z}
\begin{cases}
    \Dist{Z}_s(Z) \doteq P(Z \mid S\!=\!s) ,\\
    \Dist{Z}^y(Z) \doteq P(Z \mid Y\!=\!y) ,\\
    \Dist{Z}_s^y(Z) \doteq P(Z \mid S\!=\!s, Y\!=\!y) .
\end{cases}
\end{equation}

\begin{table*}[!t]
    \centering
    \caption{\label{tab:fairness_definitions}
        Seven Group Fairness Notions Are Considered. They All Take Values In $[0,1]$. 0 Means Absence of Discrimination and 1 Means Maximum Discrimination.}
    \adjustbox{width=1.0\textwidth}{
    \begin{tabular}{p{2cm}|p{2cm}|p{5.2cm}|p{12cm}}
        \toprule
        Category                    & Interpretation                            & Fairness Measure                                                 & Definition \\ \midrule \midrule
        Independence                & $\hat{Y} \perp S$                         & Statistical Parity (SP)~\cite{calders_DMKD_2010}                   & $\text{SP}(\hat{Y},S) \hspace{7.2ex} =\big|P(\hat{Y}=1|S=1)-P(\hat{Y}=1|S=0)\big|$ \\ \midrule
        \multirow{4}{*}{Separation} & \multirow{4}{*}{$\hat{Y} \perp S \mid Y$} & Disparity of Opportunity (DOpp)~\cite{hardt_NeurIPS_2016_equality} & $\text{DOpp}(\hat{Y},Y,S) \hspace{2ex} = \big|P(\hat{Y}=1 \mid Y=1, S=1) - P(\hat{Y}=1 \mid Y=1, S=0)\big| $ \\         \cmidrule{3-4}
                                    &                                           & Disparity of Regret (DR)~\cite{hardt_NeurIPS_2016_equality}        & $\text{DR}(\hat{Y},Y,S) \hspace{4.3ex} = \big|P(\hat{Y}=1 \mid Y=0, S=1) - P(\hat{Y}=1 \mid Y=0, S=0)\big| $ \\        \cmidrule{3-4}
                                    &                                           & Disparity of Odds (DOdds)~\cite{hardt_NeurIPS_2016_equality}      & $\text{DOdds}(\hat{Y},Y,S) \hspace{.9ex} = \frac{1}{2} \times (DOpp(\hat{Y},Y,S) + DR(\hat{Y},Y,S) ) $ \\        \midrule
        \multirow{4.5}{*}{Calibration}                 
        & \multirow{4.5}{*}{$Y \perp S \mid \hat{Y}$}              
        & Disparity of Positive Calibration (DPC)
        & $\text{DPC}(h, Y, S) \hspace{3.05ex} = \frac{1}{2}\sum_{t\in[0,1]} \big| P(Y=1, h(Z)=t\mid S=1) -  P(Y=1, h(Z)=t\mid S=0) \big| $ \\
        \cmidrule{3-4}
        & & Disparity of Negative Calibration (DNC)
        & $\text{DNC}(h, Y, S) \hspace{2.75ex} = \frac{1}{2}\sum_{t\in[0,1]} \big| P(Y=0, h(Z)=t\mid S=1) -  P(Y=0, h(Z)=t\mid S=0) \big| $ \\
        \cmidrule{3-4}
        & & Disparity of Calibration (DC)
        & $\text{DC}(h, Y, S) \hspace{4.75ex} = \frac{1}{2}\times (DPC(h, Y, S) + DNC(h, Y, S))$\\
        \bottomrule
    \end{tabular}}
    \vspace{-1.5ex}
\end{table*}

\vspace{1ex}
\noindent\textbf{Problem Statement}
We assume the data owner has a dataset $(X,S)$.
They computes a fair representation $Z=f(X)$ and shares only $Z$ to the downstream users.
The user is interested in performing a specific prediction task, which is equivalently defined by the target label $Y$, and produces prediction $\hat{Y}\sim Bern(h(Z))$.
The problem is, the data owner wishes to learn an encoding function $f$ so that predictions based on $Z$ is guaranteed to be fair for any $Y$ and $\hat{Y}$.

Two comments are in order.
First, the data owner can use $S$ to learn the encoding function $f$, but does not directly input $S$ to $f$ so $S$ is not needed at test time.
Second, the user does not necessarily need to have $Y$ linked to every $Z$.
He only need to link part of the acquired dataset with target label $Y$, using which to develop $h$.
For example, for digital footprint data and financial default prediction, personal traits such as social security number (SSN), name, and email can be used to link the record to the loan loss databases.

For clarity, we present the theoretical results in a binary setting.
We show in the supplementary material that they naturally extend to categorical variables.

\subsection{Representation Fairness and Discriminativeness}
\label{sec:PC_PD_def}
We define fairness and discriminativeness of the representation using Total Variation distance (TVD)~\cite{levin_2017_markov_chains}.

\begin{defn}[Total Variation Distance] \label{df:TVD}
Let $\Dist{P},\Dist{Q} \in D(\Omega)$ be two distributions and $D(\Omega)$ denotes the set of all distribution over sample space $\Omega$. The Total Variation distance is
$$
    d_{TV} (\Dist{P},\Dist{Q}) = \frac{1}{2} \sum_{e \in \Omega} \Big| \Dist{P}(e) - \Dist{Q}(e) \Big|.
$$
\end{defn}

TVD takes values in $[0,1]$.
A small TVD means two distributions are close and thus the information about which group a sample comes from is anonymized (from a population sense).
This motivates the following definition of representation fairness.

\begin{defn}[Representation Fairness]
The representation $Z$ is said to be $\alpha$-fair w.r.t. the sensitive attribute $S$ if
$$
d_{TV}(\Dist{Z}_0,\Dist{Z}_1) \leq \alpha.
$$
We call $\alpha$ the fairness coefficient.
\end{defn}

In contrast, a large TVD means two distributions are distinguishable and thus, given a sample, we can predict which group it comes from with high accuracy.
This motivates the following definition of representation discriminativeness.
\begin{defn}[Representation Discriminativeness]
    The representation $Z$ is said to be $\beta$-discriminative w.r.t. the target attribute $Y$ if
    $$
    1 - d_{TV}(\Dist{Z}^0,\Dist{Z}^1) \leq \beta.
    $$
We call $\beta$ the discriminativeness coefficient.
\end{defn}

Both $\alpha$ and $\beta$ take values in $[0,1]$.
Smaller $\alpha$ or $\beta$ indicates the representation is more fair or more discriminative.

Our definitions of representation fairness and discriminativeness are natural and arguably the correct formalizations.
Since TVD is equivalent to the Bayes-optimal balanced accuracy, an alternative intepretation is that fairness requires the sensitive attribute to be unpredictable and discriminativeness requires the target attribute to be predictable.

\begin{defn}[Balanced Accuracy]
    Given the prediction $\hat{Y}$ and the target label $Y$, the balanced accuracy is
    {\small
    $$
    BA(\hat{Y},Y) = 0.5 \times \left( P(\hat{Y}=0 \mid Y=0) + P(\hat{Y}=1 \mid Y=1) \right).
    $$
    }
    For a predictor $h$, we also write 
    {\small$$BA(h,Y) = \mathbb{E}_{\hat{Y}\sim Bern(h(Z))} BA(\hat{Y},Y).$$}
\end{defn}

\begin{prop} \label{prop:fair_eq_BA}
A representation $Z$ is $\alpha$-fair is equivalent to
$$
\mathbb{E}_{(Z,S)\sim \mathcal{Z}} BA(h^*(Z),S) \leq \frac{1}{2} + \frac{1}{2}\alpha,
$$
where $\mathcal{Z}$ is enriched to denote the joint ditribution of $Z$ and $S$, and $h^*(Z)= \arg\max_{h:\R^{|Z|}\rightarrow \{0,1\}} BA(h(Z),S)$ is the Bayes-optimal predictor.
\end{prop}

\begin{prop} \label{prop:disc_eq_BA}
A representation $Z$ is $\beta$-discriminative is equivalent to,
$$
\mathbb{E}_{(Z,Y)\sim \mathcal{Z}} BA(h^*(Z),Y) \geq 1 - \frac{1}{2}\beta,
$$
where $\mathcal{Z}$ is enriched to denote the joint ditribution of $Z$ and $Y$, and $h^*(Z)= \arg\max_{h:\R^{|Z|}\rightarrow \{0,1\}} BA(h(Z),Y)$ is the Bayes-optimal predictor.
\end{prop}

Throughout the paper, we will measure prediction performance using balanced accuracy because it takes a balanced view between different sensitive groups.

\subsection{Group Fairness} \label{sec:group_fairness}

We adapt the taxonomy proposed in~\cite{liu2019implicit,barocas-hardt-narayanan} to divide group fairness into three categories, namely \textit{independence}, \textit{separation}, and \textit{calibration}.
Briefly, independence requires the prediction to be independent of the sensitive attribute ($\hat{Y}\perp S$).
Separation allows such dependence as long as it can be justified by the target ($\hat{Y}\perp S\mid Y$).
And calibration indicates that after knowing the prediction, the sensitive attribute should not give more information about the target ($Y\perp S \mid \hat{Y}$).

As listed in \tabref{tab:fairness_definitions}, we consider seven notions that cover all three categories.
SP, DOpp, DR, and DOdds are commonly used fairness notions~\cite{calders_DMKD_2010,hardt_NeurIPS_2016_equality,narayanan2018translation}.
They measure different statistical disparities between sensitive groups.
DPC, DNC, and DC are further concerned with the distribution of the prediction score $h(Z)$ and interpret uncalibratedness as unfairness.
These measures all take values in $[0,1]$.
0 means perfect fairness and 1 means maximally unfair.
We refer readers to~\cite{barocas-hardt-narayanan,narayanan2018translation} for a review of different fairness notions.

To explain the motivation for DPC, DNC, and DC, we first consider maximum uncalibration (MUC)~\cite{barocas-hardt-narayanan}, which is defined below.

\begin{defn}[Maximum Uncalibration]
{\small
\begin{align*}
\text{MUC}(h,Y,S) = \max_{t\in [0,1]}\big|& P(Y=1\mid h(Z)=t, S=1)\\
& - P(Y=1\mid h(Z)=t, S=0) \big|
\end{align*}
}
\end{defn}

Achieving small MUC is undoubtedly desirable.
But we show that MUC is too strong as a fairness notion that cannot be guaranteed even with a perfectly fair and perfectly discriminative representation.
The reason is that MUC measures uncalibratedness at a score $t$ but without considering the size of the population receiving the score $t$.
Thus, the score only need to be uncalibrated at one point with infinitesimally small population to have $MUC(h,Y,S)=1$.

\begin{thm} \label{thm:negative_MUC}
    Even for a balanced dataset with $a=b=0.5$ and arbitrary $r\in (0,1)$, there exists a representation $Z$ that is $0$-fair and $0$-discriminative, and an accurate predictor $h$ with $BA(h,Y)=1$, such that $\text{MUC}(h,Y,S)=1$.
\end{thm}

Thus, we are motivated to study DPC, DNC, and DC, which are proposed as relaxations of MUC.
DPC and DNC measure the disparity of the event $\{Y=1,h(Z)=t\}$ and $\{Y=0,h(Z)=t\}$, instead of $\{Y=1\mid h(Z)=t\}$ as in MUC.
They additionally take into account the population size receiving the score $t$.
DC is taken to be the mean of DPC and DNC.

However, if the target label $Y$ itself is uncalibrated, DPC, DNC, and DC all bear an intrinsic lower bound because the dataset user cannot modify $Y$.

\begin{thm} \label{thm:DPC}
Consider sensitive attribute $S$ and target attribute $Y$ with base rates $a$ and $b$. For any predictor $h$,
$$
\text{DPC}(h,Y,S), \text{DNC}(h,Y,S), \text{DC}(h,Y,S) \geq \frac{1}{2}|a-b|.
$$
\end{thm}

Thus, our objective for DPC, DNC, and DC is not to achieve 0 but to achieve $\frac{1}{2}|a-b|$. Intuitively, our prediction should not be more uncalibrated than the target label $Y$.

Lastly, it is known in literature~\cite{mcnamara_arxiv_2017} that SP can be directly upper bounded by the fairness coefficient $\alpha$.
The known upper bounds on DOpp, DR, and DOdds~\cite{madras_ICML_2018} are invalid in our setting because they require using the target labels, which we assume unavailable.





\section{Guaranteeing Perfect Fairness}
\label{sec:perfect_fairness}
We first show that if the representation is 0-fair and 0-discriminative, all considered group fairness notions are guaranteed to be 0.

\begin{thm} \label{thm:exact_all}
Consider sensitive attribute $S$ and target attribute $Y$ with base rates $a$, $b$, $r$.
If the representation $Z$ is 0-fair and 0-discriminative, for any $Y$ and $\hat{Y}$ (equivalently any $h$), we have
{\small
\begin{align*}
&SP(\hat{Y},S) = DOpp(\hat{Y},Y,S)=DR(\hat{Y},Y,S)=DOdds(\hat{Y},Y,S)\\
=&DPC(h,Y,S)=DNC(h,Y,S)=DC(h,Y,S)=0.
\end{align*}}
\end{thm}

\begin{proof}[\thref{thm:exact_all}]
The representation $Z$ being 0-fair and 0-discriminative implies the following,
{\small
\begin{equation*}
\begin{cases}
a -b = 0,\\
d_{TV}(\mathcal{Z}_0,\mathcal{Z}_1) = 0, \\
d_{TV}(\mathcal{Z}_0^y,\mathcal{Z}_1^y) = 0, ~\forall y\in\{0,1\}, \\
d_{TV}(\mathcal{Z}^0,\mathcal{Z}^1) = 1, \\
d_{TV}(\mathcal{Z}_s^0,\mathcal{Z}_s^1) = 1, ~\forall s\in\{0,1\}.\\
\end{cases}
\end{equation*}}

These conditions mean that (1) different sensitive groups must have equal positive rate ($a=b$) and (2) 0-fair and 0-discriminative representation anonymizes $S$ not only at the population level but also in the subgroups defined by $Y$.
They ensure that all considered group fairness notions are 0.
\end{proof}

One thing to note is that \thref{thm:exact_all} does not have conflict with \thref{thm:DPC} and the fairness impossibility results in the literature~\cite{chouldechova2017fair,barocas-hardt-narayanan}.
The twist is that 0-fair and 0-discriminative representation is only possible when $a=b$, where both \thref{thm:DPC} and the fairness impossibility result become unrestrictive.
We discuss further in~\secref{sec:discussion}.

\thref{thm:exact_all} justifies the use of a fair and discriminative representation: when achieved exactly, it guarantees perfect fairness for all seven group fairness notions. 

\section{Guaranteeing Approximate Fairness}
\label{sec:approximate_fairness}
Next we proceed to show that $\alpha$-fair and $\beta$-discriminative representations guarantee all seven group fairness notions approximately.
Notably, it is known in literature~\cite{mcnamara_arxiv_2017} that $SP$ is tightly upper bounded by $\alpha$.
Thus, in the following we only consider the other six fairness notions.
Complete proofs are deferred to the supplementary material.

\subsection{Results on DOpp, DR, DOdds}
We view the problem of guaranteeing approximate DOpp, DR, or DOdds for a $\alpha$-fair and $\beta$-discriminative representation as a constrained optimization problem.
We maximize DOpp, DR, or DOdds over all distributions of $Z$ subject to representation fairness and discriminativeness constraints, and all predictor $h$, see \eqref{eq:OP}.
The solution to this problem gives a tight upper bound on the respective fairness measure for any prediction on any representation that is $\alpha$-fair and $\beta$-discriminative.

The original problem, which we shorthand as \textbf{$OP(OBJ)$} with $OBJ$ denote the maximand (DOpp, DR, or DOdds), is as follows:
{\small
\begin{subequations} \label{eq:OP}
\begin{align}
\max_{\mathcal{Z}_0^0, \mathcal{Z}_0^1, \mathcal{Z}_1^0, \mathcal{Z}_1^1,h} &
~ OBJ(h;\mathcal{Z}_0^0, \mathcal{Z}_0^1, \mathcal{Z}_1^0, \mathcal{Z}_1^1), \\
s.t.\qquad &\mathcal{Z}_0^0, \mathcal{Z}_0^1, \mathcal{Z}_1^0, \mathcal{Z}_1^1 \in D(\R^{|Z|}) \label{eq:OP.a} \\
& h:\mathbb{R}^{|Z|}\rightarrow [0,1] \\
& d_{TV}(\mathcal{Z}_0, \mathcal{Z}_1) \leq \alpha,  \label{eq:OP.b}\\
& 1- d_{TV}(\mathcal{Z}^0, \mathcal{Z}^1) \leq \beta, \label{eq:OP.c}
\end{align}
\end{subequations}
}
where $D(\mathbb{R}^{|Z|})$ denotes the set of all distributions over $\mathbb{R}^{|Z|}$, and $\alpha$ and $\beta$ are the fairness and discriminativeness coefficients.
We make it explicit that the value of $OBJ$ is determinied by both the predictor $h$ and the distributions of $Z$ in different subgroups, \ie~$\mathcal{Z}_0^0, \mathcal{Z}_0^1, \mathcal{Z}_1^0, \mathcal{Z}_1^1$.
Other distributions, such as $\mathcal{Z}_0$, $\mathcal{Z}_1$, $\mathcal{Z}^0$, and $\mathcal{Z}^1$, can be expressed as follows,
{\small
\begin{equation} \label{eq:all_Zs}
\begin{cases}
  \Dist{Z}_0 = a\Dist{Z}_0^1 + (1-a)\Dist{Z}_0^0 , \\
  \Dist{Z}_1 = b\Dist{Z}_1^1 + (1-b)\Dist{Z}_1^0 , \\
  \Dist{Z}^0 = \frac{(1-b)r}{(1-a)(1-r)+(1-b)r}\Dist{Z}_1^0 + \Big(1-\frac{(1-b)r}{(1-a)(1-r)+(1-b)r}\Big)\Dist{Z}_0^0 , \\
  \Dist{Z}^1 = \frac{br}{a(1-r)+br}\Dist{Z}_1^1 + \Big(1-\frac{br}{a(1-r)+br}\Big)\Dist{Z}_0^1,
\end{cases}
\end{equation}
}
where $a$, $b$, and $r$ are the population base rates defined in~\eqref{eq:base_rates}.

The way we solve \textbf{$OP(OBJ)$} is by considering another problem, which we call the reduced problem and denote as \textbf{$RP1(OBJ)$}:
{\small
\begin{subequations} \label{eq.RP}
\begin{align}
\max_{\mathcal{Z}_0^0, \mathcal{Z}_0^1, \mathcal{Z}_1^0, \mathcal{Z}_1^1} ~& OBJ(h;\mathcal{Z}_0^0, \mathcal{Z}_0^1, \mathcal{Z}_1^0, \mathcal{Z}_1^1), \label{eq.RP.a}\\
s.t.\qquad~& \mathcal{Z}_0^0, \mathcal{Z}_0^1, \mathcal{Z}_1^0, \mathcal{Z}_1^1 \in D(\{0,1\}^{3}), \label{eq.RP.b}\\
& h( i,j,k)=i,  ~\forall i,j,k\in\{0,1\}, \label{eq.RP.i} \\
& d_{TV}(\mathcal{Z}_0, \mathcal{Z}_1) \leq \alpha, \label{eq.RP.c}\\
& 1- d_{TV}(\mathcal{Z}^0, \mathcal{Z}^1) \leq \beta, \label{eq.RP.d} \\
& \hspace{-10ex} (-1)^{j+1}(\mathcal{Z}_1(i,j,k) - \mathcal{Z}_0(i,j,k) ) \geq 0, ~\forall i,j,k \in\{0,1\}, \label{eq.RP.g} \\
& \hspace{-10ex} (-1)^{k+1}(\mathcal{Z}^1(i,j,k) - \mathcal{Z}^0(i,j,k) ) \geq 0, ~\forall i,j,k \in\{0,1\}, \label{eq.RP.h} 
\end{align}
\end{subequations}}
where we use $(i,j,k)$ to denote a 3-dimension binary representation.
With a slight abuse of notation, we use $\mathcal{Z}_1^0(i,j,k)$ as a shorthand for $\mathcal{Z}_1^0((i,j,k))=P(Z=(i,j,k)\mid S=1,Y=0)$, and similarly for other distributions.
We also write $h(i,j,k)$ as a shorthand for $h((i,j,k))$.

The reduced problem differs from the original problem in the following ways.
First, instead of considering representation $Z$ from an infinite representation space $\mathbb{R}^{|Z|}$, we only consider a finite representation space $\{0,1\}^3$ (\eqref{eq.RP.b}), \ie~3-dimension binary representations.
We denote such a representation as $(i,j,k)$.
An important consequence is that now each of $\mathcal{Z}_0^0$, $\mathcal{Z}_0^1$, $\mathcal{Z}_1^0$, $\mathcal{Z}_1^1$ can be expressed as $2^3=8$ non-negative variables that sum to 1.
Second, instead of considering randomized predictors, we consider deterministic and fixed predictors (\eqref{eq.RP.i}).
For any representation $(i,j,k)$, $h$ is fixed to predict $i$, which is either 0 or 1.
Third, we additionally impose \eqref{eq.RP.g} and (\ref{eq.RP.h}), which we call the positivity constraints.
These two constraints allow us to convert $d_{TV}(\mathcal{Q}_1,\mathcal{Q}_0)$ and $d_{TV}(\mathcal{Q}^0,\mathcal{Q}^1)$ into linear expressions.\footnote{Note that TVD is the sum of the absolute differences. The positivity constraints determine the sign of the differences at every representation $(i,j,k)$, using which we can convert TVD into a linear expression.}

The following claim tells us that $OP(OBJ)$ reduces to $RP1(OBJ)$ for $OBJ=DOpp$, $DR$, and $DOdds$.
Thus, we can solve $RP1(OBJ)$ to obtain upper bounds on $DOpp$, $DR$, and $DOdds$.

\begin{claim} \label{claim:OP_RP1}
$OP(OBJ) = RP1(OBJ)$ for $OBJ=DOpp$, $DR$, and $DOdds$.
\end{claim}

\begin{proof}(sketch)
We start with the observation that for $OBJ=DOpp$, $DR$, and $DOdds$, for any $\mathcal{Z}_0^0, \mathcal{Z}_0^1, \mathcal{Z}_1^0, \mathcal{Z}_1^1$, deterministic predictor $h$ always maximize $OBJ(h;\mathcal{Z}_0^0, \mathcal{Z}_0^1, \mathcal{Z}_1^0, \mathcal{Z}_1^1)$.
Thus, we only need to show the deterministic version of $OP(OBJ)$, which we denote as $OP'(OBJ)$, is equivalent to $RP1(OBJ)$.

To prove $OP'(OBJ)=RP1(OBJ)$, we show both $OP'(OBJ) \geq RP1(OBJ)$ and $OP'(OBJ) \leq RP1(OBJ)$. 
The former is true because $RP1(OBJ)$ is strictly more restrictive than $OP'(OBJ)$: any feasible solution to $RP1(OBJ)$ is also feasible in $OP'(OBJ)$.
The latter can be shown by a construction that finds solutions of $RP1(OBJ)$ that achieves the same value of $OBJ$ as solutions of $OP'(OBJ)$.

The intuition behind our proof is, although $RP1(OBJ)$ is strictly more restrictive than $OP(OBJ)$, it still contains all extreme values of $OBJ$, for $OBJ=DOpp$, $DR$, and $DOdds$.
\end{proof}

Furthermore, $RP1(OBJ)$ is a linear program for $OBJ=DOpp$, $DR$, and $DOdds$, where all constraints are linear and the objective function involves absolute values.
Using the standard linear reformulation that expresses absolute values as multiple linear expressions, $RP1(OBJ)$ can be efficiently solved using simplex method~\cite{dantzig1998linear} or interior-point method~\cite{andersen2000mosek}.

\begin{lemma} \label{lma:RP1_is_linear}
In $RP1(OBJ)$, each of $\mathcal{Z}_0^0$, $\mathcal{Z}_0^1$, $\mathcal{Z}_1^0$, $\mathcal{Z}_1^1$ can be expressed as $2^3=8$ non-negative variables that sum to 1.
There are in total $4\times 8 =32$ independent variables.
For $OBJ=DOpp$, $DR$, and $DOdds$, $RP1(OBJ)$ is a linear program in these variables with an objective function that involves absolute values.
\end{lemma}

\begin{proof}[~\lemref{lma:RP1_is_linear}]
The positivity constraints \eqref{eq.RP.g} and (\ref{eq.RP.h}) are inherently linear.
Using \eqref{eq.RP.g} and (\ref{eq.RP.h}), \eqref{eq.RP.c} and (\ref{eq.RP.d}) also become linear.
The last step is to show $OBJ(h;\mathcal{Z}_0^0, \mathcal{Z}_0^1, \mathcal{Z}_1^0, \mathcal{Z}_1^1)$ is a linear function coupled with absolute values.
We show this for $OBJ=DOpp$,
{\small
\begin{align*}
DOpp(h;\mathcal{Z}_0^0, \mathcal{Z}_0^1, \mathcal{Z}_1^0, \mathcal{Z}_1^1) 
=& |\sum_{Z\in \{0,1\}^3} h(Z)(\mathcal{Z}_1^1(Z) - \mathcal{Z}_0^1(Z))| \\
=& | \sum_{j,k \in \{0,1\}} (\mathcal{Z}_1^1(1,j,k) - \mathcal{Z}_0^1(1,j,k)) |. 
\end{align*}}
The same can be shown for $DR$ and $DOdds$.
\end{proof}

\subsection{Results on DPC, PNC, DC}
For DPC, DNC, and DC, we reduce $OP(OBJ)$ to another problem, which we denote as $RP2(OBJ)$:
{\small
\begin{subequations}  \label{eq.RP2}
\begin{align}
\max_{\mathcal{Z}_0^0, \mathcal{Z}_0^1, \mathcal{Z}_1^0, \mathcal{Z}_1^1} ~& \sup_{h:\mathbb{R}^{|Z|}\rightarrow [0,1]} OBJ(h;\mathcal{Z}_0^0, \mathcal{Z}_0^1, \mathcal{Z}_1^0, \mathcal{Z}_1^1), \label{eq.RP2.a}\\
s.t.\qquad~& \mathcal{Z}_0^0, \mathcal{Z}_0^1, \mathcal{Z}_1^0, \mathcal{Z}_1^1 \in D(\{0,1\}^{4}), \label{eq.RP2.b}\\
& d_{TV}(\mathcal{Z}_0, \mathcal{Z}_1) \leq \alpha, \label{eq.RP2.c}\\
& 1- d_{TV}(\mathcal{Z}^0, \mathcal{Z}^1) \leq \beta, \label{eq.RP2.d} \\
& (-1)^{i+1}(\mathcal{Z}_1(i,j,k,l) - \mathcal{Z}_0(i,j,k,l) ) \geq 0, \nonumber \\
& \hspace{15ex}  ~\forall i,j,k,l \in\{0,1\}, \label{eq.RP2.g} \\
& (-1)^{j+1}(\mathcal{Z}^1(i,j,k,l) - \mathcal{Z}^0(i,j,k,l) ) \geq 0, \nonumber \\
& \hspace{15ex}  ~\forall i,j,k,l \in\{0,1\}, \label{eq.RP2.h} \\
& \hspace{-13ex} (-1)^{k+1}((1-b)\mathcal{Z}_1^0(i,j,k,l) - (1-a)\mathcal{Z}_0^0(i,j,k,l) ) \geq 0,  \nonumber \\
& \hspace{15ex}  ~\forall i,j,k,l \in\{0,1\}, \label{eq.RP2.e} \\
& \hspace{-1ex} (-1)^{l+1}(b\mathcal{Z}_1^1(i,j,k,l) - a\mathcal{Z}_0^1(i,j,k,l) ) \geq 0, \nonumber \\
& \hspace{15ex} ~\forall i,j,k,l \in\{0,1\}, \label{eq.RP2.f}
\end{align}
\end{subequations}}

In $RP2(OBJ)$, we consider 4-dimension binary representations, which we denote as $(i,j,k,l)$.
Instead of maximizing $OBJ$ over $\mathcal{Z}_0^0, \mathcal{Z}_0^1, \mathcal{Z}_1^0, \mathcal{Z}_1^1,h$, we maximize the supremum of $OBJ$ over $\mathcal{Z}_0^0, \mathcal{Z}_0^1, \mathcal{Z}_1^0, \mathcal{Z}_1^1$.
This is because although DPC, DNC, DC themselves are difficult to analyze, their supremum (over all $h$) can be easily computed.
We also additionally impose \eqref{eq.RP2.e} and (\ref{eq.RP2.f}), using which $\sup_{h:\mathbb{R}^{|Z|}\rightarrow [0,1]} OBJ(h;\mathcal{Z}_0^0, \mathcal{Z}_0^1, \mathcal{Z}_1^0, \mathcal{Z}_1^1)$ can be linearized.

The supremum (over all $h$) for DPC, DNC, and DC has the following expression.

\begin{lemma}
For any $\mathcal{Z}_0^0$, $\mathcal{Z}_0^1$, $\mathcal{Z}_1^0$, $\mathcal{Z}_1^1$,
{\small
\begin{align*}
\sup_{h:\mathbb{R}^{|Z|}\rightarrow [0,1]} DPC
=& \frac{1}{2} \sum_{Z} |b \mathcal{Z}_1^1(Z) - a\mathcal{Z}_0^1(Z)|,\\
\sup_{h:\mathbb{R}^{|Z|}\rightarrow [0,1]} DNC 
=& \frac{1}{2} \sum_{Z} |(1-b) \mathcal{Z}_1^0(Z) - (1-a)\mathcal{Z}_0^0(Z)|,\\
\sup_{h:\mathbb{R}^{|Z|}\rightarrow [0,1]} DC
=& \frac{1}{2} \sum_{Z} |b \mathcal{Z}_1^1(Z) - a\mathcal{Z}_0^1(Z)| \\
& + \frac{1}{2} \sum_{Z} |(1-b) \mathcal{Z}_1^0(Z) - (1-a)\mathcal{Z}_0^0(Z)|.
\end{align*}}
Furthermore, all supremum is achieved when $h(Z)$ is different for different $Z$.
\end{lemma}

Applying the same analysis from \claimref{claim:OP_RP1} and \lemref{lma:RP1_is_linear} but with $\sup_{h:\mathbb{R}^{|Z|}\rightarrow [0,1]} OBJ(h;\mathcal{Z}_0^0, \mathcal{Z}_0^1, \mathcal{Z}_1^0, \mathcal{Z}_1^1)$ as the maximand,
$OP(OBJ)$ also reduces to $RP2(OBJ)$, which is a linear program.
Thus, we can efficiently solve $RP2(OBJ)$ to obtain tight upper bounds on DPC, DNC, and DC.

\begin{claim} \label{claim:OP_RP2}
$OP(OBJ) = RP2(OBJ)$ for $OBJ=DPC$, $DNC$, and $DC$.
\end{claim}

\begin{lemma} \label{lma:RP2_is_linear}
In $RP2(OBJ)$, each of $\mathcal{Z}_0^0$, $\mathcal{Z}_0^1$, $\mathcal{Z}_1^0$, $\mathcal{Z}_1^1$ can be expressed as $2^4=16$ non-negative variables that sum to 1.
There are in total $4\times 16 =64$ independent variables.
For $OBJ=DPC$, $DNC$, and $DC$, $RP2(OBJ)$ is a linear program in these variables.
\end{lemma}

The results in this section provide another justification for the use of a both fair and discriminative representation: when achieved approximately, \ie~with a $\alpha$-fair and $\beta$-discriminative representation, all seven group fairness notions can be guaranteed approximately.
Moreover, the obtained upper bounds are tight because we prove equivalence between $OP(OBJ)$ and $RP1(OBJ)$ / $RP2(OBJ)$.
Finding these quantitative fairness guarantees simply reduces to solving linear programs.

\section{Discussions on Fair Representation}
\label{sec:discussion}

\subsection{Guaranteeing Group Fairness}
\begin{table}[t]
    \centering
    \caption{Summarization of Our Theoretical Results. $RP1(\cdot)$ and $RP2(\cdot)$ Are Linear Programs That Depend on $\alpha$ and $\beta$, Defined In \eqref{eq.RP} and (\ref{eq.RP2}), Respectively.}
    \label{tab:fairness_bounds}
    \adjustbox{width=1.0\columnwidth}{{\small
    \begin{tabular}{p{2.4cm} p{2.8cm}  p{2.8cm}  c  c}
        \toprule
        \multirow{3.5}{*}{Fairness Notion}    & \multicolumn{2}{c}{Fairness Guarantee} & \multirow{3.5}{*}{Tight} & \multirow{3.5}{*}{Reference} \\
        \cmidrule{2-3}
        & using $0$-fair and $0$-disc. representation &  using $\alpha$-fair and $\beta$-disc. representation \\
        \midrule
        \midrule
        $\text{SP}(\hat{Y},S)$         & 0 &  $\alpha$ & \checkmark & ~\cite{feldman_SIGKDD_2015,mcnamara_arxiv_2017}  \\
        \midrule
        $\text{DOpp}(\hat{Y},Y,S)$     & 0 &  $RP1(DOpp)$ & \checkmark & This work \\
        \midrule
        $\text{DR}(\hat{Y},Y,S)$       & 0 &  $RP1(DR)$ & \checkmark & This work \\
        \midrule
        $\text{DOdds}(\hat{Y},Y,S)$    & 0 &  $RP1(DOdds)$ & \checkmark & This work \\
        \midrule
        $\text{DPC}(h, Y, S)$          & 0 &  $RP2(DPC)$  & \checkmark & This work \\
        \midrule
        $\text{DNC}(h, Y, S)$          & 0 &  $RP2(DNC)$  &  \checkmark & This work \\
        \midrule
        $\text{DC}(h, Y, S)$          & 0 &  $RP2(DC)$  &  \checkmark & This work \\
        \bottomrule
      \end{tabular}}}
      \vspace{-1ex}
\end{table}
We summarize our theoretical results in~\tabref{tab:fairness_bounds}.
Our findings justify the use of fair representation (both in the exact and approximate sense), \emph{for all prediction tasks that the representation is discriminative for}.
For the other prediction tasks for which the representation is undiscriminative, the downstream predictions can still be unfair.
But this is less of a compromise.
For these prediction tasks, the predictions cannot be accurate in the first place and thus are unlikely to be actually deployed.
Thus, discriminativeness establishes a dichotomy between the prediction tasks whose fairness can be guaranteed using fair representation and those that cannot.

An immediate corollary of our theoretical results is that the downstream predictions' fairness, for all seven group fairness notions, is linearly controlled by $\alpha$ and $\beta$.
Thus, using an approximately fair and approximately discriminative representation degrades the fairness guarantees gracefully. 

\begin{corrolary}
For any base rates $a$, $b$, $r$, and for $OBJ=SP,DOpp,DR,DOdds,DPC,DNC$, or $DC$, using a representation that is $\alpha$-fair and $\beta$-discriminative, we have
$$
OBJ \leq C \max(\alpha,\beta),
$$
where $C$ is a constant that depends on $a$, $b$, $r$, and the concerned fairness notion.
\end{corrolary}

\begin{proof}
It directly follows from the fact that (1) 0-fair and 0-discriminative representations guarantee seven group fairness notions exactly and (2) $\alpha$-fair and $\beta$-discriminative representations guarantee seven group fairness notions via linear programs.
\end{proof}

Our theoretical results also indicate that representation discriminativeness and downstream predictions' fairness are aligned.
By encouraging fair representation to encode more information about the data, better fairness guarantees can be obtained for a larger class of prediction tasks.

\subsection{Fairness Impossibility Theorem}
A well-known result in fair machine learning is the fairness impossibility theorem~\cite{chouldechova2017fair,kleinberg2018inherent,barocas-hardt-narayanan}, which states that any two fairness notions cannot both hold under mild conditions.

\begin{thm}[Fairness Impossibility Theorem~\cite{barocas-hardt-narayanan}]
\label{th:fairness_impossibility}
Consider three kinds of fairness, \emph{independence} ($\hat{Y}\perp S$), \emph{separation} ($\hat{Y}\perp S \mid Y$), and \emph{sufficiency} ($Y\perp S \mid \hat{Y}$). When the base rate differ (\ie~$a\neq b$), any two out of three cannot both hold.
\end{thm}

Since we show that 0-fair and 0-discriminative representation guarantees multiple group fairness notions exactly, it is natural to ask how it relates to the impossibility theorem.
We make two comments.

\textbf{First, fair representation does not violate the impossibility theorem.}
Since the impossibility theorem reveals the inherent imcompatibility between different fairness notions, achieving them using fair representation cannot do any better.
In fact, both 0-fair and 0-discriminative representation is feasible \emph{exactly when the impossibility theorem becomes unrestrictive, \ie~when $a=b$}.

\begin{thm}[Representation Fairness-Discriminativeness Trade-off] \label{thm:tradeoff}
    For base rates $a$, $b$, $r$, and for any representation that is $\alpha$-fair and $\beta$-discriminative,
    {\small
    \begin{equation}
    \begin{aligned}
    \beta \geq &\frac{|a-b|^2}{ r\big(1-r\big) \big(\frac{a}{r}+\frac{b}{1-r}\big) \big(\frac{1-a}{r}+\frac{1-b}{1-r}\big) } \\
    &- \frac{|a-b|}{ r\big(1-r\big) \big(\frac{a}{r}+\frac{b}{1-r}\big) \big(\frac{1-a}{r}+\frac{1-b}{1-r}\big) } \alpha.
    \end{aligned}
    \end{equation}}
\end{thm}

\thref{thm:tradeoff} shows that there is a linear trade-off between $\alpha$ and $\beta$, depending on the base rates $a$, $b$, $r$.
When the target and sensitive attributes are coupled ($a\neq b$), the first term in the RHS becomes positive and prohibits $\alpha$ and $\beta$ to be both zero.
When $a=b$, the trade-off vanishes.
Thus, in the context of fair representation, the incompatibility between different fairness notions manifests as a trade-off between representation fairness and discriminatveness.

\thref{thm:tradeoff} should be interpreted positively because (1) the incompatibility is inherent in the fairness notions rather than specific to fair representation, and (2) fair representation is indeed able to achieve seven group fairness notions exactly whenever unrestricted by the impossibility theorem.

\textbf{Second, fair representation generalizes the impossibility theorem.}
The impossibility theorem only states that, when $a\neq b$, achieving multiple fairness notions \emph{exactly} is impossible.
A natural question (raised in Barocas~\etal~\cite{barocas-hardt-narayanan}) is whether meaningful trade-off can be achieved for \emph{approximate} fairness notions.
In the real-world, we are often more interested in requiring approximate instead of exact fairness, in order for important business goals to proceed.
Our results answer this question in the affirmative and can be seen as a generalization of the impossibility theorem.
Whenever restricted by the impossibility theorem, fair representation provides a concrete way to guarantees approximate fairness for all notions.

\section{Learning Fair and Discriminative Representation} 
\label{sec:learning_general_fair_representation}

\subsection{Learning Discriminative Representation}
\label{sec:learn_PD_representation}

Let's first consider learning discriminative representations.
The difficulty in our setting is, since the data owner does not know the downstream prediction tasks, they cannot directly use the target labels to learn discriminative representations.
Thus, we rely on self-supervised representation learning.
We set up a pretext loss, such as reconstruction loss or contrastive loss~\cite{deng_CVPR_2019_arcface}, to self-supervise the representation learning process.
Here, we do not overly emphasise the performance of the pretext loss.
Rather, we hope the process of solving the pretext loss guide the representation to summarize important semantics from the data.

We acknowledge that self-supervised representation learning is itself an active research area and different data typically requires different pretext loss.
Thus, we do not propose nor fix a pretext loss.
We assume different pretext loss will be used for different data, potentially advised by domain experts.

\subsection{Learning Fair Representation}
\label{sec:learn_PC_representation}

To learn fair representations, the ultimate goal is to minimize the Total Variation Distance (TVD).
However, directly minimizing TVD is impractical because (1) the gradient is 0, \ie~uninformative, when two distributions have disjoint support, and
(2) TVD cannot be easily estimated from finite samples.
Our motivation is that TVD is an instance of the Integral Probability Metrics (IPMs), with witness function class $\mathcal{F}_{TVD}=\{f:\|f\|_{\infty}\leq 1\}$.

\begin{defn}[Integral Probability Metrics~\cite{muller_1997_IPMs}]
	Let $\Dist{P},\Dist{Q}\in D(\Omega)$ and $\mathcal{F}$ be a class of real-valued bounded measurable functions on $\Omega$. The integral probability metric with witness function class $\mathcal{F}$ is defined as
	$$
	IPM_{\mathcal{F}}(\Dist{P},\Dist{Q})=\sup_{f\in \mathcal{F}}\left|\int_{\Omega} fd\Dist{P} - \int_{\Omega} fd\Dist{Q}\right|.
	$$
\end{defn}

By changing the witness function class to 
$$
\mathcal{F}_{MMD}=\{f:\|f\|_{\mathcal{H}}\leq 1\},
$$
where $\mathcal{H}$ is a reproducing kernel Hilbert space (RKHS) with reproducing kernel $k$, we obtain an alternative for TVD, the Maximum Mean Discrepancy (MMD).

\begin{defn}[Maximum Mean Discrepancy~\cite{gretton_JMLR_2012_mmd}]
	$$
	d_{MMD}(\Dist{P},\Dist{Q})=IPM_{\mathcal{F}_{MMD}}(\Dist{P},\Dist{Q}).
	$$
\end{defn}

MMD is desirable for several reasons.
First, MMD provides informative gradient even when two distributions have disjoint support.
Second, minimizing MMD is asymptotically equivalent to minimizing TVD because MMD attains the minimum, 0, iff two distributions are equivalent (provided that the RKHS $\mathcal{H}$ is universal)~\cite{steinwart2001influence}.
Third, MMD is computationally cheap 
and enjoys a faster convergence rate compared to other IPMs (see Corollary 3.5 in~\cite{sriperumbudur_EJS_2012}).
Using the kernel trick, the squared MMD has an unbiased finite-sample estimate.

\begin{defn}[Unbiased Estimate of Squared MMD~\cite{gretton_JMLR_2012_mmd}] \label{def:unbiased_MMD}
	Given $m$ i.i.d. samples $(x_1,...,x_m)\sim \Dist{P}^{m}$ and $n$ i.i.d. samples $(y_1,...,y_n)\sim \Dist{Q}^{n}$, an unbiased estimate of the squared MMD with kernel $k$ is
	{\small
	\begin{align*}
		\tilde{d}^2_{MMD} &(\Dist{P},\Dist{Q}) = \frac{1}{m(m-1)}\sum_{i=1}^{m}\sum_{j\neq i}^{m} k(x_i,x_j)\\
		&+ \frac{1}{n(n-1)}\sum_{i=1}^{n}\sum_{j\neq i}^{n} k(y_i,y_j) - \frac{2}{mn}\sum_{i=1}^{m}\sum_{j=i}^{n} k(x_i,y_j).
	\end{align*}
	}
\end{defn}

\begin{defn}[Gaussian/Rational Quadratic Kernel]
	{\small
	\begin{align*}
	k_{G}(x,y)& = \exp \bigg(-\frac{\|x-y\|^2}{2l_{g}^2} \bigg), \\
	k_{RQ}(x,y)& = \bigg(1+\frac{\|x-y\|^2}{2\alpha l_{rq}^2} \bigg)^{-\alpha}.
	\end{align*}
	}
\end{defn}

Gaussian kernel is the kernel of choice in MMD~\cite{gretton_JMLR_2012_mmd}.
However, setting the lengthscale $l_{g}$ requires heuristics~\cite{ramdas_AAAI_2015}.
In light of this, we use rational quadratic kernel, which equals an infinite sum of the Gaussian kernels with the inversed squared lengthscale {\small $1/l_{g}^2$} following a Gamma distribution of shape parameter $\alpha$ and mean {\small$1/l_{rq}^2$}~\cite{rasmussen_2003_GPML}.
Theoretically, the rational quadratic kernel is less sensitive to lengthscale selection compared to the Gaussian kernel.

\subsection{Two-Player Game Formulation}

\secref{sec:learn_PD_representation} and \ref{sec:learn_PC_representation} have explained the motivation of using pretext loss and MMD to learn fair and discriminative representations.
We then cast learning as a constrained optimization problem:
\begin{equation} \label{eq:constraint_problem}
\begin{aligned}
\min_{f\in \mathcal{F}} \quad &L_p(f, \{x_i, s_i\}^n_{i=1}), \\
s.t.\quad& \tilde{d}^2_{MMD}(f,\{x_i, s_i\}^n_{i=1}) \leq \epsilon,
\end{aligned}
\end{equation}
where $\mathcal{F}$ is a class of encoding functions, $L_{p}$ is the pretext loss, $\tilde{d}^2_{MMD}$ is the MMD constraint, $\{x_i, s_i\}^n_{i=1}$ is the finite samples of $(X,S)$, and $\epsilon$ is a small tolerance.

The equivalent Lagrangian dual problem is
\begin{equation} \label{eq:minimax_problem}
\begin{aligned}
\min_{f\in \mathcal{F}}\max_{\lambda\geq 0}\quad & L_p(f, \{x_i, s_i\}^n_{i=1}) \\
&+ \lambda \times (\tilde{d}^2_{MMD}(f,\{x_i, s_i\}^n_{i=1}) - \epsilon),
\end{aligned}
\end{equation}
where $\lambda$ is a non-negative weight.

Motivated by the close connection between online decision problem and game theory~\cite{freund1999adaptive,zinkevich2003online}, we propose to solve~\eqref{eq:minimax_problem} using a two-player game, shown in~\algoref{algo:learning_representation}.
The game is formulated between a learner and a regulator.
At each round $t$, a non-negative weight $\lambda$ is maintained.
The learner conducts regularized loss minimization (RLM), where MMD is used as a fair regularizer.
The regulator increases $\lambda$ if $\tilde{d}_{MMD}^2 > \epsilon$ and vice versa.

\SetCommentSty{mycommfont}

\SetKwInput{KwInput}{Input}                
\SetKwInput{KwOutput}{Output}              
\SetKwInput{KwInitialize}{Initialize}              
\begin{algorithm}[t]
\DontPrintSemicolon

  \KwInput{pretext loss $L_{p}$, fair regularization $\tilde{d}^2_{MMD}$, representation function class $\mathcal{F}$, \# rounds $T$, tolerance $\epsilon$}
  \KwData{$\{x_i,s_i\}_{i=1}^{n}$}
  \KwInitialize{$\lambda = 1$}
  \For{$t=1:T$}    
    { 
      \tcc{Learner conducts regularized loss minimization.}
      $
      f_t=\arg\min_{f\in\mathcal{F}}L_p(f, \{x_i, s_i\}^n_{i=1}) + \lambda\times \tilde{d}^2_{MMD}(f,\{x_i, s_i\}^n_{i=1})
      $\;
      \tcc{Regulator adjusts the weight.}
      \uIf{$\tilde{d}^2_{MMD}(f,\{x_i, s_i\}^n_{i=1}) > \epsilon$}{
        increase $\lambda$ \;
      }
      \Else{
        decrease $\lambda$ \;
      }
    }
  \KwOutput{$h_{T}$}

\caption{Learning Fair \& Disc. Representation}
\label{algo:learning_representation}
\end{algorithm}

We make the following remark, where the learner plays a simpler strategy, to justify the design of \algoref{algo:learning_representation}.

\begin{remark}[last-iterate convergence in convex-concave minimax optimization~\cite{golowich2020last}] \label{remark:covergence}
If both $L_p$ and $\tilde{d}_{MMD}^2$ are convex and smooth in the parameters of $f$ (thus the objective of \eqref{eq:minimax_problem} is convex-concave), the learner can play extragradient descent and the regulator can play extragradient ascent to approximately solve~\eqref{eq:minimax_problem} in last-iterate.
\end{remark}

However, in practice we typically instantiate $f$ with a neural network, resulting in a nonconvex-concave optimization.
Thus, we ask the learner to conduct RLM rather than a one-step (extra)gradient update.
In our implementation, the regulator plays a coarse-to-fine line search strategy.
We show in supplementary material that our formulation is rather robust to $\lambda$ and thus the simple line search strategy works well.
We set $\epsilon$ to $0$ (note that {\small $\tilde{d}^2_{MMD}(f,\{x_i, s_i\}^n_{i=1})$} is an unbiased estimate and can be negative).

\section{Experiments}
\label{sec:experiment}

\begin{table*}[t]
    \definecolor{green}{rgb}{0,0.45,0}
    \centering
    \caption{Our Approach on Adult Dataset Compared with the Baseline (\ie~Representations Learned with Only Reconstruction) and Existing Fair Representation Learning Methods. Numbers in the Brackets (\textcolor{green}{green}) Are the Fairness Guarantees Computed From Our Theoretical Results. $\uparrow$/$\downarrow$ Means Higher/Lower Is Better.}
    \label{tab:Adult}
    \vspace{-1ex}
    \setlength\extrarowheight{-1pt}
    \adjustbox{width=1.0\textwidth}{
    \begin{tabular}{l c c c c l l l l l l l}
        \toprule
        \multirow{2}{*}{Methods} &
        Sen. Attr. ($\downarrow$) & &
        Tar Attr. ($\uparrow$) & &
        \multicolumn{7}{c}{Fairness ($\downarrow$)} \\
        \cmidrule{2-2} \cmidrule{4-4} \cmidrule{6-12}
        &  BA   & & BA  & & SP & DOpp & DR & DOdds & DPC & DNC & DC\\ 
        \midrule
        Reconstruction only
        & .846{\small \textcolor{gray}{$\pm$.003}} & &
        .811{\small \textcolor{gray}{$\pm$.002}} & &
        .177{\small \textcolor{gray}{$\pm$.011}} \textcolor{green}{(.692)} &
        .096{\small \textcolor{gray}{$\pm$.025}} \textcolor{green}{(1.)} &
        .108{\small \textcolor{gray}{$\pm$.010}} \textcolor{green}{(1.)} &
        .102{\small \textcolor{gray}{$\pm$.017}} \textcolor{green}{(1.)} &
        .103{\small \textcolor{gray}{$\pm$.002}} \textcolor{green}{(.219)} &
        .252{\small \textcolor{gray}{$\pm$.019}} \textcolor{green}{(.781)} &
        .178{\small \textcolor{gray}{$\pm$.010}} \textcolor{green}{(.5)}
        \\
        \midrule
        
        Beutel et al.~\cite{beutel_arxiv_2017}
        & .827{\small \textcolor{gray}{$\pm$.012}} & &
        \textbf{.800{\small \textcolor{gray}{$\pm$.003}}} & &
        .187{\small \textcolor{gray}{$\pm$.081}} \textcolor{green}{(.654)} &
        .172{\small \textcolor{gray}{$\pm$.114}} \textcolor{green}{(1.)} &
        .118{\small \textcolor{gray}{$\pm$.052}} \textcolor{green}{(1.)} &
        .145{\small \textcolor{gray}{$\pm$.081}} \textcolor{green}{(1.)} &
        \textbf{.099{\small \textcolor{gray}{$\pm$.032}} \textcolor{green}{(.219)}} &
        .246{\small \textcolor{gray}{$\pm$.057}} \textcolor{green}{(.781)} &
        .173{\small \textcolor{gray}{$\pm$.039}} \textcolor{green}{(.5)}
        \\
        
        LFR~\cite{zemel_ICML_2013_LFR}
        & .667{\small \textcolor{gray}{$\pm$.167}} & &
        .648{\small \textcolor{gray}{$\pm$.148}} & &
        .097{\small \textcolor{gray}{$\pm$.080}} \textcolor{green}{(.334)} &
        .141{\small \textcolor{gray}{$\pm$.059}} \textcolor{green}{(1.)} &
        .058{\small \textcolor{gray}{$\pm$.054}} \textcolor{green}{(.740)} &
        .099{\small \textcolor{gray}{$\pm$.007}} \textcolor{green}{(.870)} &
        .122{\small \textcolor{gray}{$\pm$.021}} \textcolor{green}{(.219)} &
        .188{\small \textcolor{gray}{$\pm$.058}} \textcolor{green}{(.553)} &
        .155{\small \textcolor{gray}{$\pm$.019}} \textcolor{green}{(.386)} 
        \\
        
        VFAE~\cite{louizos_arxiv_2015}
        & .745{\small \textcolor{gray}{$\pm$.124}} & &
        .727{\small \textcolor{gray}{$\pm$.114}} & &
        .145{\small \textcolor{gray}{$\pm$.015}} \textcolor{green}{(.490)} &
        .105{\small \textcolor{gray}{$\pm$.029}} \textcolor{green}{(1.)} &
        .091{\small \textcolor{gray}{$\pm$.010}} \textcolor{green}{(.917)} &
        .098{\small \textcolor{gray}{$\pm$.016}} \textcolor{green}{(.959)} &
        .106{\small \textcolor{gray}{$\pm$.007}} \textcolor{green}{(.219)} &
        .247{\small \textcolor{gray}{$\pm$.025}} \textcolor{green}{(.709)} &
        .176{\small \textcolor{gray}{$\pm$.015}} \textcolor{green}{(.464)}
        \\
        
        FFVAE~\cite{creager_ICML_2019}
        & .628{\small \textcolor{gray}{$\pm$.006}} & &
        .610{\small \textcolor{gray}{$\pm$.006}} & &
        .044{\small \textcolor{gray}{$\pm$.012}} \textcolor{green}{(.256)} &
        .055{\small \textcolor{gray}{$\pm$.022}} \textcolor{green}{(1.)} &
        .026{\small \textcolor{gray}{$\pm$.010}} \textcolor{green}{(.651)} &
        .040{\small \textcolor{gray}{$\pm$.015}} \textcolor{green}{(.826)} &
        .100{\small \textcolor{gray}{$\pm$.001}} \textcolor{green}{(.219)} &
        .124{\small \textcolor{gray}{$\pm$.010}} \textcolor{green}{(.475)} &
        .112{\small \textcolor{gray}{$\pm$.005}} \textcolor{green}{(.347)}
        \\
        \midrule



        Ours ({\small $l_{rq}=2$})
        & .611{\small \textcolor{gray}{$\pm$.040}} & &
        .733{\small \textcolor{gray}{$\pm$.042}} & &
        .056{\small \textcolor{gray}{$\pm$.036}} \textcolor{green}{(.222)} &
        \textbf{.035{\small \textcolor{gray}{$\pm$.018}} \textcolor{green}{(1.)}} &
        .021{\small \textcolor{gray}{$\pm$.020}} \textcolor{green}{(.612)} &
        .028{\small \textcolor{gray}{$\pm$.014}} \textcolor{green}{(.806)} &
        .101{\small \textcolor{gray}{$\pm$.002}} \textcolor{green}{(.219)} &
        .119{\small \textcolor{gray}{$\pm$.025}} \textcolor{green}{(.441)} &
        .110{\small \textcolor{gray}{$\pm$.013}} \textcolor{green}{(.330)}
        \\

        Ours ({\small $l_{rq}=2\sqrt{2}$})
        & \textbf{.573{\small \textcolor{gray}{$\pm$.029}}} & &
        .681{\small \textcolor{gray}{$\pm$.038}} & &
        \textbf{.029{\small \textcolor{gray}{$\pm$.022}} \textcolor{green}{(.146)}} &
        .038{\small \textcolor{gray}{$\pm$.030}} \textcolor{green}{(1.)} &
        \textbf{.010{\small \textcolor{gray}{$\pm$.008}} \textcolor{green}{(.526)}} &
        \textbf{.024{\small \textcolor{gray}{$\pm$.018}} \textcolor{green}{(.763)}} &
        .100{\small \textcolor{gray}{$\pm$.001}} \textcolor{green}{(.219)} &
        \textbf{.104{\small \textcolor{gray}{$\pm$.004}} \textcolor{green}{(.365)}} &
        \textbf{.102{\small \textcolor{gray}{$\pm$.002}} \textcolor{green}{(.292)}}
        \\
        \bottomrule
    \end{tabular}}
\end{table*}

We conduct experiments to investigate both our theoretical results and the proposed fair representation learning approach.

\subsection{Datasets}

We experiment on five datasets, namely \textbf{Adult}, \textbf{MPI3D}, \textbf{VGGFace2}, \textbf{Labeled Faces in the Wild (LFW)}, and \textbf{CelebA}. 
Sample images are shown in \figref{fig:dataset_examples}.

\textbf{Adult}~\cite{kohavi_KDD_1996} is a tabular dataset extracted from the U.S. census data. 
It contains attributes such as age, gender, occupation, and education.
The prediction task is income ($\geq$50K or $<$50K) and the sensitive attribute is gender.
We randomly split data into training, validation, and test sets, which contain 13,602, 13,602, and 13,603 instances, respectively.

\textbf{MPI3D}~\cite{gondal_NeurIPS_2019_MPI3D} contains images captured in a controlled platform, where an object is mounted on the tip of a manipulator.
We preprocess the labels so that the four attributes-of-interest become binary.
They are $\text{object color}\!\in\!\{\text{white}, \text{red}\}$, $\text{shape}\!\in\!\{\text{cube}, \text{pyramid}\}$, $\text{size}\!\in\!\{\text{small}, \text{big}\}$, and $\text{background color}\!\in\!\{\text{purple}, \text{sea green}\}$.
Object color is regarded as the sensitive attribute.
We sample 10\% of the data, \ie~7,680 images, each for training/validation/test set.

\textbf{VGGFace2}~\cite{cao_FG_2018_VGGFace2} is a face image dataset with large inter- and intra-identity variations.
Gender is the sensitive attribute.
We follow the preprocessing in \cite{liu2017sphereface} to obtain normalized face images.
The training set contains 3,140,709 images of 8631 identities.
We randomly sample 50 images per identity (total of 500 identities) from the original test set as our test set.
\textbf{LFW}~\cite{huang_2008_LFW} is a benchmark face image dataset.
We use the same preprocessing to obtain normalized face images.
We only use the test set to evaluate on the face verification task.

Since both VGGFace2 and LFW are without facial attributes, we additionally evaluate on \textbf{CelebA} dataset~\cite{liu2015deep}, where each image is labelled with 40 attributes including gender. We apply the same preprocessing to normalize face images.
The training and evaluation set contains 162,770 and 19,867 images, respectively.

\begin{figure}[!t]
    \begin{minipage}{1.0\linewidth}
        \begin{minipage}{1.0\linewidth}
        \centering
            \includegraphics[width=0.13\linewidth]{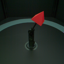}
            \includegraphics[width=0.13\linewidth]{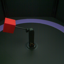}
            \includegraphics[width=0.13\linewidth]{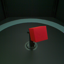}
            \includegraphics[width=0.13\linewidth]{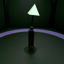}
            \includegraphics[width=0.13\linewidth]{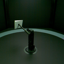}
            \includegraphics[width=0.13\linewidth]{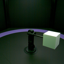}
            \includegraphics[width=0.13\linewidth]{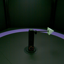}
        \end{minipage}\vspace{0.5ex}\\
        \begin{minipage}{1.0\linewidth}
        \centering
            \includegraphics[width=0.13\linewidth]{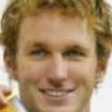}
            \includegraphics[width=0.13\linewidth]{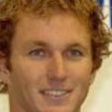}
            \includegraphics[width=0.13\linewidth]{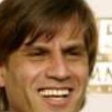}
            \includegraphics[width=0.13\linewidth]{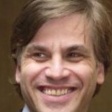}
            \includegraphics[width=0.13\linewidth]{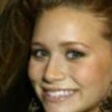}
            \includegraphics[width=0.13\linewidth]{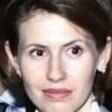}
            \includegraphics[width=0.13\linewidth]{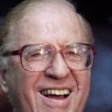}
        \end{minipage}\vspace{0.5ex}\\
        \begin{minipage}{1.0\linewidth}
        \centering
            \includegraphics[width=0.13\linewidth]{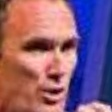} 
            \includegraphics[width=0.13\linewidth]{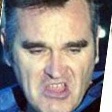}
            \includegraphics[width=0.13\linewidth]{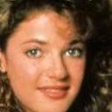}
            \includegraphics[width=0.13\linewidth]{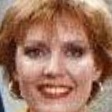} 
            \includegraphics[width=0.13\linewidth]{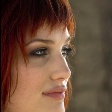}
            \includegraphics[width=0.13\linewidth]{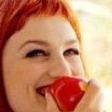}
            \includegraphics[width=0.13\linewidth]{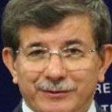}
        \end{minipage} \vspace{0.5ex} \\
        \begin{minipage}{1.0\linewidth}
        \centering
            \includegraphics[width=0.13\linewidth]{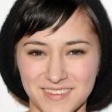} 
            \includegraphics[width=0.13\linewidth]{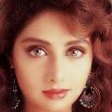}
            \includegraphics[width=0.13\linewidth]{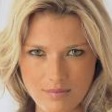}
            \includegraphics[width=0.13\linewidth]{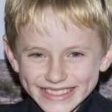} 
            \includegraphics[width=0.13\linewidth]{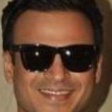}
            \includegraphics[width=0.13\linewidth]{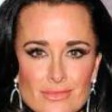}
            \includegraphics[width=0.13\linewidth]{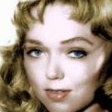}
        \end{minipage}
    \end{minipage}
    \caption{Sample images (from top to bottom) of MPI3D~\cite{gondal_NeurIPS_2019_MPI3D}, LFW~\cite{huang_2008_LFW}, VGGFace2~\cite{cao_FG_2018_VGGFace2}, and CelebA~\cite{liu2015deep}.}
    \label{fig:dataset_examples}
    \vspace{-2.5ex}
\end{figure}

\subsection{Evaluation}

We simulate the dataset owner-user scenario in our experiments.
We first learn representation using pretext loss, without knowing what prediction tasks the representation will be used for.
We then train predictors on top of the learned representation without any fairness constraint.
We evaluate on four aspects: the fairness and discriminativeness of the learned representation, the fairness of downstream predictions, and the theoretical fairness guarantees.

The \textbf{representation fairness} is evaluated by posing an adversary to predict the sensitive attribute.
We directly report the balanced accuracy, the lower the fairer, instead of the fairness coefficient $\alpha$.
Since $\alpha$ is equivalent to the \emph{Bayes-optimal} balanced accuracy, stronger adversary will produce a more accurate (and honest) evaluation. 
Our considered adversary is a 7-layer fully-connected network (7-Net) with width 8$\times$ the input dimension and skip-layer connections.
We validate in the supplementary material that 7-Net is considerably stronger than the adversarial models considered in the fair representation literature~\cite{feldman_SIGKDD_2015,xie_NeurIPS_2017,louizos_arxiv_2015}. 
Deeper and wider networks only have marginal improvements.

We use the same 7-Net as downstream predictors.
The \textbf{representation discriminativeness} is evaluated by the downstream predictions' balanced accuracy, which is an estimate of the discriminativeness coefficient $\beta$.
In order to evaluate if the learned representation leads to fair outcomes, we report \textbf{the predictions' fairness} w.r.t. seven group fairness notions (see~\tabref{tab:fairness_definitions}).
We also report the \textbf{theoretical fairness guarantees}, which are obtained by solving linear programs defined by the base rates $a$, $b$, $r$, and $\alpha$, $\beta$.


\subsection{Experiment on Adult}
\begin{table*}[t]
    \definecolor{green}{rgb}{0,0.45,0}
    \centering
    \caption{Our Approach on Uncorrelated MPI3D Dataset Compared with Representations Learned with Only Reconstruction and Existing Fair Representation Learning Methods. The Sens. Attr. is Color. Notations are Consistent with \tabref{tab:Adult}.}
    \label{tab:MPI3D_uncorrelated}
    \adjustbox{width=1\textwidth}{
    \begin{tabular}{l || l || l c l l l l l l l}
        \toprule
        
        \multirow{2}{*}{Methods} & \multirow{2}{*}{\makecell[t]{Sen. Attr.'s\\BA ($\downarrow$)}}  & \multirow{2}{*}{Tar. Attr.} & \multirow{2}{*}{BA ($\uparrow$)} 
        & \multicolumn{6}{c}{Fairness ($\downarrow$)} \\ 
        \cmidrule{5-11}
        & & & & SP & DOpp & DR & DOdds & DPC & DNC & DC \\
        \midrule
        \multirow{3}{*}{Reconstruction only} 
        & \multirow{3}{*}{.935{\small \textcolor{gray}{$\pm$.021}}} 
        & Shape
        & .804{\small \textcolor{gray}{$\pm$.014}} 
        & .053{\small \textcolor{gray}{$\pm$.040}} \textcolor{green}{(.870)} & 
        .064{\small \textcolor{gray}{$\pm$.036}} \textcolor{green}{(1.)} & 
        .056{\small \textcolor{gray}{$\pm$.037}} \textcolor{green}{(1.)} &
        .060{\small \textcolor{gray}{$\pm$.033}} \textcolor{green}{(1.)} &
        .081{\small \textcolor{gray}{$\pm$.018}} \textcolor{green}{(.5)} &
        .074{\small \textcolor{gray}{$\pm$.020}} \textcolor{green}{(.5)} &
        .077{\small \textcolor{gray}{$\pm$.015}} \textcolor{green}{(.5)} \\
        & & Size
        & .805{\small \textcolor{gray}{$\pm$.018}} 
        & .094{\small \textcolor{gray}{$\pm$.042}} \textcolor{green}{(.870)} & 
        .113{\small \textcolor{gray}{$\pm$.053}} \textcolor{green}{(1.)} &
        .074{\small \textcolor{gray}{$\pm$.034}} \textcolor{green}{(1.)} &
        .094{\small \textcolor{gray}{$\pm$.042}} \textcolor{green}{(1.)} &
        .113{\small \textcolor{gray}{$\pm$.053}} \textcolor{green}{(.5)} &
        .085{\small \textcolor{gray}{$\pm$.035}} \textcolor{green}{(.5)} &
        .099{\small \textcolor{gray}{$\pm$.043}} \textcolor{green}{(.5)} \\
        & & Bkgd color
        & .986{\small \textcolor{gray}{$\pm$.008}} 
        & .018{\small \textcolor{gray}{$\pm$.028}} \textcolor{green}{(.870)} & 
        .047{\small \textcolor{gray}{$\pm$.035}} \textcolor{green}{(1.)} & 
        .034{\small \textcolor{gray}{$\pm$.020}} \textcolor{green}{(1.)} & 
        .040{\small \textcolor{gray}{$\pm$.026}} \textcolor{green}{(.898)} &
        .059{\small \textcolor{gray}{$\pm$.033}} \textcolor{green}{(.5)} &
        .040{\small \textcolor{gray}{$\pm$.021}} \textcolor{green}{(.5)} &
        .049{\small \textcolor{gray}{$\pm$.025}} \textcolor{green}{(.449)} \\
        \midrule

        \multirow{3}{*}{Beutel et al.~\cite{beutel_arxiv_2017}} 
        & \multirow{3}{*}{.918{\small \textcolor{gray}{$\pm$.057}}} 
        & Shape
        & .732{\small \textcolor{gray}{$\pm$.025}} 
        & .077{\small \textcolor{gray}{$\pm$.027}} \textcolor{green}{(.836)} 
        & .052{\small \textcolor{gray}{$\pm$.027}} \textcolor{green}{(1.)} 
        & .107{\small \textcolor{gray}{$\pm$.048}} \textcolor{green}{(1.)} 
        & .080{\small \textcolor{gray}{$\pm$.026}} \textcolor{green}{(1.)} 
        & .109{\small \textcolor{gray}{$\pm$.035}} \textcolor{green}{(.5)}
        & .142{\small \textcolor{gray}{$\pm$.051}} \textcolor{green}{(.5)}
        & .126{\small \textcolor{gray}{$\pm$.041}} \textcolor{green}{(.5)} \\
        & & Size 
        & .767{\small \textcolor{gray}{$\pm$.025}} 
        & .094{\small \textcolor{gray}{$\pm$.039}} \textcolor{green}{(.836)} 
        & .143{\small \textcolor{gray}{$\pm$.066}} \textcolor{green}{(1.)} 
        & .065{\small \textcolor{gray}{$\pm$.033}} \textcolor{green}{(1.)} 
        & .104{\small \textcolor{gray}{$\pm$.042}} \textcolor{green}{(1.)} 
        & .167{\small \textcolor{gray}{$\pm$.074}} \textcolor{green}{(.5)}
        & .115{\small \textcolor{gray}{$\pm$.041}} \textcolor{green}{(.5)}
        & .141{\small \textcolor{gray}{$\pm$.055}} \textcolor{green}{(.5)}\\
        & & Bkgd color
        & .936{\small \textcolor{gray}{$\pm$.027}} 
        & .013{\small \textcolor{gray}{$\pm$.012}} \textcolor{green}{(.836)} 
        & .089{\small \textcolor{gray}{$\pm$.058}} \textcolor{green}{(1.)} 
        & .070{\small \textcolor{gray}{$\pm$.060}} \textcolor{green}{(1.)} 
        & .079{\small \textcolor{gray}{$\pm$.058}} \textcolor{green}{(.964)} 
        & .102{\small \textcolor{gray}{$\pm$.066}} \textcolor{green}{(.5)}
        & .090{\small \textcolor{gray}{$\pm$.059}} \textcolor{green}{(.5)}
        & .096{\small \textcolor{gray}{$\pm$.060}} \textcolor{green}{(.482)}\\
        \midrule

        \multirow{3}{*}{LFR~\cite{zemel_ICML_2013_LFR}} 
         & \multirow{3}{*}{.809{\small \textcolor{gray}{$\pm$.102}}} 
        & Shape
        & .726{\small \textcolor{gray}{$\pm$.078}} 
        & .087{\small \textcolor{gray}{$\pm$.049}} \textcolor{green}{(.618)} 
        & .084{\small \textcolor{gray}{$\pm$.042}} \textcolor{green}{(1.)} 
        & .096{\small \textcolor{gray}{$\pm$.050}} \textcolor{green}{(1.)} 
        & .090{\small \textcolor{gray}{$\pm$.045}} \textcolor{green}{(1.)} 
        & .083{\small \textcolor{gray}{$\pm$.036}} \textcolor{green}{(.5)}
        & .105{\small \textcolor{gray}{$\pm$.040}} \textcolor{green}{(.5)}
        & .094{\small \textcolor{gray}{$\pm$.037}} \textcolor{green}{(.5)} \\
        & & Size 
        & .733{\small \textcolor{gray}{$\pm$.071}} 
        & .125{\small \textcolor{gray}{$\pm$.068}} \textcolor{green}{(.618)} 
        & .141{\small \textcolor{gray}{$\pm$.074}} \textcolor{green}{(1.)} 
        & .109{\small \textcolor{gray}{$\pm$.064}} \textcolor{green}{(1.)} 
        & .125{\small \textcolor{gray}{$\pm$.068}} \textcolor{green}{(1.)} 
        & .115{\small \textcolor{gray}{$\pm$.050}} \textcolor{green}{(.5)}
        & .104{\small \textcolor{gray}{$\pm$.047}} \textcolor{green}{(.5)}
        & .110{\small \textcolor{gray}{$\pm$.048}} \textcolor{green}{(.5)}\\
        & & Bkgd color
        & .909{\small \textcolor{gray}{$\pm$.061}} 
        & .005{\small \textcolor{gray}{$\pm$.005}} \textcolor{green}{(.618)} 
        & .052{\small \textcolor{gray}{$\pm$.031}} \textcolor{green}{(1.)} 
        & .045{\small \textcolor{gray}{$\pm$.026}} \textcolor{green}{(1.)} 
        & .049{\small \textcolor{gray}{$\pm$.028}} \textcolor{green}{(.8)} 
        & .052{\small \textcolor{gray}{$\pm$.023}} \textcolor{green}{(.5)}
        & .050{\small \textcolor{gray}{$\pm$.023}} \textcolor{green}{(.5)}
        & .051{\small \textcolor{gray}{$\pm$.021}} \textcolor{green}{(.4)}\\
        \midrule

        \multirow{3}{*}{VFAE~\cite{louizos_arxiv_2015}} 
        & \multirow{3}{*}{.836{\small \textcolor{gray}{$\pm$.062}}} 
        & Shape
        & .668{\small \textcolor{gray}{$\pm$.059}} 
        & .122{\small \textcolor{gray}{$\pm$.027}} \textcolor{green}{(.672)} 
        & .094{\small \textcolor{gray}{$\pm$.024}} \textcolor{green}{(1.)} 
        & .149{\small \textcolor{gray}{$\pm$.039}} \textcolor{green}{(1.)} 
        & .122{\small \textcolor{gray}{$\pm$.027}} \textcolor{green}{(1.)} 
        & .072{\small \textcolor{gray}{$\pm$.009}} \textcolor{green}{(.5)}
        & .108{\small \textcolor{gray}{$\pm$.023}} \textcolor{green}{(.5)}
        & .090{\small \textcolor{gray}{$\pm$.014}} \textcolor{green}{(.5)}\\
        & & Size 
        & .715{\small \textcolor{gray}{$\pm$.049}} 
        & .160{\small \textcolor{gray}{$\pm$.040}} \textcolor{green}{(.672)} 
        & .204{\small \textcolor{gray}{$\pm$.052}} \textcolor{green}{(1.)} 
        & .116{\small \textcolor{gray}{$\pm$.037}} \textcolor{green}{(1.)} 
        & .160{\small \textcolor{gray}{$\pm$.040}} \textcolor{green}{(1.)} 
        & .085{\small \textcolor{gray}{$\pm$.021}} \textcolor{green}{(.5)}
        & .076{\small \textcolor{gray}{$\pm$.018}} \textcolor{green}{(.5)}
        & .081{\small \textcolor{gray}{$\pm$.017}} \textcolor{green}{(.5)}\\
        & & Bkgd color
        & .821{\small \textcolor{gray}{$\pm$.034}} 
        & .012{\small \textcolor{gray}{$\pm$.008}} \textcolor{green}{(.672)} 
        & .113{\small \textcolor{gray}{$\pm$.058}} \textcolor{green}{(1.)} 
        & .093{\small \textcolor{gray}{$\pm$.074}} \textcolor{green}{(1.)} 
        & .103{\small \textcolor{gray}{$\pm$.066}} \textcolor{green}{(1.)} 
        & .058{\small \textcolor{gray}{$\pm$.028}} \textcolor{green}{(.5)}
        & .071{\small \textcolor{gray}{$\pm$.046}} \textcolor{green}{(.5)}
        & .064{\small \textcolor{gray}{$\pm$.037}} \textcolor{green}{(.5)}\\
        \midrule

        \multirow{3}{*}{FFVAE~\cite{creager_ICML_2019}} 
        & \multirow{3}{*}{.958{\small \textcolor{gray}{$\pm$.013}}} 
        & Shape
        & .671{\small \textcolor{gray}{$\pm$.024}} 
        & .026{\small \textcolor{gray}{$\pm$.020}} \textcolor{green}{(.916)} 
        & .025{\small \textcolor{gray}{$\pm$.017}} \textcolor{green}{(1.)} 
        & .064{\small \textcolor{gray}{$\pm$.025}} \textcolor{green}{(1.)} 
        & .044{\small \textcolor{gray}{$\pm$.014}} \textcolor{green}{(1.)} 
        & \textbf{.037{\small \textcolor{gray}{$\pm$.007}} \textcolor{green}{(.5)}}
        & .060{\small \textcolor{gray}{$\pm$.014}} \textcolor{green}{(.5)}
        & .049{\small \textcolor{gray}{$\pm$.009}} \textcolor{green}{(.5)} \\
        & & Size 
        & .684{\small \textcolor{gray}{$\pm$.019}} 
        & .022{\small \textcolor{gray}{$\pm$.014}} \textcolor{green}{(.916)} 
        & .026{\small \textcolor{gray}{$\pm$.017}} \textcolor{green}{(1.)} 
        & .029{\small \textcolor{gray}{$\pm$.022}} \textcolor{green}{(1.)} 
        & .027{\small \textcolor{gray}{$\pm$.012}} \textcolor{green}{(1.)} 
        & \textbf{.037{\small \textcolor{gray}{$\pm$.006}} \textcolor{green}{(.5)}}
        & .039{\small \textcolor{gray}{$\pm$.010}} \textcolor{green}{(.5)}
        & .038{\small \textcolor{gray}{$\pm$.004}} \textcolor{green}{(.5)}\\
        & & Bkgd color
        & .686{\small \textcolor{gray}{$\pm$.027}} 
        & .014{\small \textcolor{gray}{$\pm$.010}} \textcolor{green}{(.916)} 
        & .232{\small \textcolor{gray}{$\pm$.060}} \textcolor{green}{(1.)} 
        & .249{\small \textcolor{gray}{$\pm$.053}} \textcolor{green}{(1.)} 
        & .241{\small \textcolor{gray}{$\pm$.055}} \textcolor{green}{(1.)} 
        & .119{\small \textcolor{gray}{$\pm$.026}} \textcolor{green}{(.5)}
        & .159{\small \textcolor{gray}{$\pm$.038}} \textcolor{green}{(.5)}
        & .139{\small \textcolor{gray}{$\pm$.031}} \textcolor{green}{(.5)}\\
        \midrule
        
        \multirow{3}{*}{Ours} 
        & \multirow{3}{*}{\textbf{.507{\small \textcolor{gray}{$\pm$.012}}}}
        & Shape
        & \textbf{.846{\small \textcolor{gray}{$\pm$.030}}}
        & \textbf{.006{\small \textcolor{gray}{$\pm$.004}} \textcolor{green}{(.014)}}
        & \textbf{.007{\small \textcolor{gray}{$\pm$.005}} \textcolor{green}{(.336)}}
        & \textbf{.008{\small \textcolor{gray}{$\pm$.005}} \textcolor{green}{(.336)}}
        & \textbf{.007{\small \textcolor{gray}{$\pm$.004}} \textcolor{green}{(.322)}} 
        & \textbf{.037{\small \textcolor{gray}{$\pm$.004}} \textcolor{green}{(.168)}}
        & \textbf{.037{\small \textcolor{gray}{$\pm$.006}} \textcolor{green}{(.168)}}
        & \textbf{.037{\small \textcolor{gray}{$\pm$.005}} \textcolor{green}{(.161)}}\\
        & & Size 
        & \textbf{.868{\small \textcolor{gray}{$\pm$.034}}} 
        & \textbf{.009{\small \textcolor{gray}{$\pm$.007}} \textcolor{green}{(.014)}} 
        & \textbf{.012{\small \textcolor{gray}{$\pm$.008}} \textcolor{green}{(.292)}} 
        & \textbf{.008{\small \textcolor{gray}{$\pm$.006}} \textcolor{green}{(.292)}} 
        & \textbf{.010{\small \textcolor{gray}{$\pm$.007}} \textcolor{green}{(.278)}} 
        & \textbf{.037{\small \textcolor{gray}{$\pm$.009}} \textcolor{green}{(.146)}}
        & \textbf{.034{\small \textcolor{gray}{$\pm$.008}} \textcolor{green}{(.146)}}
        & \textbf{.036{\small \textcolor{gray}{$\pm$.008}} \textcolor{green}{(.139)}}\\
        & & Bkgd color
        & \textbf{.961{\small \textcolor{gray}{$\pm$.042}}} 
        & \textbf{.002{\small \textcolor{gray}{$\pm$.002}} \textcolor{green}{(.014)}}
        & \textbf{.004{\small \textcolor{gray}{$\pm$.003}} \textcolor{green}{(.106)}} 
        & \textbf{.003{\small \textcolor{gray}{$\pm$.002}} \textcolor{green}{(.106)}} 
        & \textbf{.004{\small \textcolor{gray}{$\pm$.002}} \textcolor{green}{(.092)}} 
        & \textbf{.019{\small \textcolor{gray}{$\pm$.011}} \textcolor{green}{(.053)}} 
        & \textbf{.018{\small \textcolor{gray}{$\pm$.012}} \textcolor{green}{(.053)}}
        & \textbf{.019{\small \textcolor{gray}{$\pm$.012}} \textcolor{green}{(.046)}}\\
        \bottomrule
    \end{tabular}}
\end{table*}

\noindent\textbf{Experimental Settings~}
We consider \textit{gender} and \textit{income} as the sensitive and target attributes, respectively.
Other attributes are posed as the features.
For representation learning, we use the same 7-Net as the network architecture and reconstruction as the pretext loss.
The representation dimension is 16 and the final weight $\lambda$ is 200.
The shape parameter for the rational quadratic kernel is 2 and the length scale is varied from $\{2,2\sqrt{2}\}$.
We train the model with SGD optimizer for 110 epochs with batch size 256.
The initial learning rate is 1 and is divided by 10 at 70th and 90th epochs.

\vspace{1ex}
\noindent\textbf{Results~}
\tabref{tab:Adult} compares our proposed approach with the baseline and existing fair representation methods.
LFR~\cite{zemel_ICML_2013_LFR} is a clustering-based method.
Beutel~\etal~\cite{beutel_arxiv_2017} uses adversarial training.
VFAE~\cite{louizos_arxiv_2015} and FFVAE~\cite{creager_ICML_2019} are VAE-based methods and derive information-theoretic regularizations.
They have been shown effective for learning fair representation when the downstream tasks are known (\ie~using the target labels) and for specific fairness notions.
We replace the prediction task therein with the same pretext loss, \ie~reconstruction loss, to adapt to our setting where target labels are unavailable.

The baseline learns discriminative but unfair representations, and thus is unable to guarantee that downstream predictions are fair.
The sensitive attribute is predictable with 84.6\% balanced accuracy and the downstream predictions are indeed shown to be unfair.
Existing methods either fail to anonymize the sensitive attribute against the postulated adversary or incur a large cost in discriminativeness.
For Beutel~\etal, LFR, and VFAE, the downstream predictions' fairness is only slightly improved.
FFVAE achieves the best fairness among existing methods, but incurs a large reduction in the target attribute's prediction accuracy.

Compared with existing fair representation learning approaches \emph{for unknown prediction tasks}, our approach learns fair representation with a smaller cost in discriminativeness.\footnote{This is unavoidable because gender and income are correlated. In Adult dataset, $|a-b|=|.316-.121|=.195$ and, by~\thref{thm:tradeoff}, it incurs a trade-off between representation fairness and discriminativeness.}
The sensitive attribute can only be predicted with 57.3\% balanced accuracy, close to random guessing (50\%).
The downstream predictions become fairer w.r.t. all fairness notions except DPC, which remains at the same level.
This is in agreement with our theoretical results because DPC, DNC, and DC are also lower bounded by $\frac{1}{2}|a-b|=.098$ (\thref{thm:DPC}) and our approach indeed achieves this lower bound.
Compared with FFVAE, which achieves the best fairness among existing methods, ours ($l_{rq}=2\sqrt{2}$) learns fairer representations, incurs a smaller cost in discriminativeness, and further improves downstream predictions' fairness.

In~\tabref{tab:Adult}, we also report the quantitative fairness guarantees computed from our theoretical results.
These guarantees are all valid and are indicative of the downstream predictions' actual fairness.
However, some of the worst-case fairness guarantees (\eg~DOpp) aren't very useful.\footnote{By useful we mean comparable to the 80\% rule recommended by US EEOC~\cite{equal_1978_department} which says a selection rate for any group which is less than 80\% of that of the highest group is generally regarded as evidence of disparate impact.}
But for DOpp, we still observe an improvement on the predictions' actual fairness.


\subsection{Experiment on MPI3D}

\noindent\textbf{Experimental Settings~}
We consider \emph{object color} as the sensitive attribute, and \emph{object shape, size}, and \emph{background color} as the target attributes.
We use ResNet-34~\cite{he_CVPR_2016_resnet} as the network architecture and reconstruction as the pretext task.
Representation dimension is 32 and the final $\lambda$ is 10.
The lengthscale and shape parameter for the rational quadratic kernel are 1 and 4, respectively.
We train the model with SGD optimizer for 150 epochs with batch size 256. 
The initial learning rate is 0.1 and is divided by 10 at 80th, 110th, and 130th epochs.

\vspace{1ex}
\noindent\textbf{Uncorrelated MPI3D~}
\tabref{tab:MPI3D_uncorrelated} shows results on uncorrelated MPI3D dataset, where the attributes are mutually independent.
For existing methods, we similarly replace the prediction task with reconstruction when learning representation.

Although the considered attributes are all mutually independent, the baseline still produces slightly discriminative predictions.
This is because it does not anonymize the sensitive attribute, which can be falsely picked up by downstream predictors.
Existing methods, similar to the findings on Adult dataset, either fail to anonymize the sensitive attribute or incur a large reduction in the target attributes' prediction accuracy.
Consequently, they also do not improve fairness for all seven fairness notions and for all three target attributes.
For example, FFVAE improves fairness for the target attributes shape and size, but fails for background color.

\begin{table*}[t]
    \centering
    \caption{Our Approach on Face Datasets Compared with the Baseline and Existing Gender-blind Face Representation Learning Methods.
    For Existing Methods, Their Gender Predictability Is Taken From the Respective Sources and Is Not Directly Comparable Due to Different Evaluation Details (See Footnote).
    }
    \label{tab:face}
    \setlength\extrarowheight{-1pt}
    \adjustbox{width=\textwidth}{
    \begin{tabular}{l c c c c c c c c c}
        \toprule
        \multirow{2}{*}{Methods} & \multicolumn{4}{c}{Face verification on LFW (\%, $\uparrow$)} & & \multicolumn{4}{c}{Face identification on VGGFace2 (\%, $\uparrow$)}\\ 
        \cmidrule{2-5} \cmidrule{7-10}
        & TAR@1e-5 & TAR@1e-4 & TAR@1e-3 & TAR@1e-2 & & Top-1 & Top-10 & Top-100 & Top-1k \\ 
        \midrule
        ArcFace only  &
        82.8{\small\textcolor{gray}{$\pm$4.0}} &
        82.7{\small\textcolor{gray}{$\pm$3.8}} & 
        94.4{\small\textcolor{gray}{$\pm$1.3}} &
        99.0{\small\textcolor{gray}{$\pm$0.2}} & & 
        49.9{\small\textcolor{gray}{$\pm$0.3}} &
        66.4{\small\textcolor{gray}{$\pm$0.2}} &
        81.5{\small\textcolor{gray}{$\pm$0.1}} &
        92.8{\small\textcolor{gray}{$\pm$0.1}} 
        \\
        \rowcolor[HTML]{D0D0D0} Ours &
        69.2{\small \textcolor{gray}{$\pm$3.5}} &
        69.4{\small \textcolor{gray}{$\pm$3.4}} &
        86.8{\small \textcolor{gray}{$\pm$2.4}} &
        96.2{\small \textcolor{gray}{$\pm$0.5}} & & 
        34.3{\small \textcolor{gray}{$\pm$1.3}} &
        50.5{\small \textcolor{gray}{$\pm$1.3}} &
        68.1{\small \textcolor{gray}{$\pm$1.2}} &
        85.4{\small \textcolor{gray}{$\pm$0.6}} \\
        \midrule
        \midrule
        \multirow{2}{*}{Methods} & \multicolumn{9}{c}{Gender prediction (\%, $\downarrow$)} \\ 
        \cmidrule{2-10}
        & Logistic Regression & Decision Tree &Random Forest & 5-Nearest Neighbors & & 7-Net (width $8\times$) &7-Net (width $16\times$) & 12-Net & SVM \\ 
        \midrule
        ArcFace only$^{*}$
        & 96.9{\small\textcolor{gray}{$\pm$0.05}}
        & 94.9{\small\textcolor{gray}{$\pm$0.2}}
        & 96.8{\small\textcolor{gray}{$\pm$0.08}}
        & 96.8{\small\textcolor{gray}{$\pm$0.09}}
        & & 97.0{\small\textcolor{gray}{$\pm$0.04}}
        & 97.0{\small\textcolor{gray}{$\pm$0.02}}
        & 97.3{\small\textcolor{gray}{$\pm$0.06}}
        & 97.3{\small\textcolor{gray}{$\pm$0.06}}
        \\
        \rowcolor[HTML]{D0D0D0} Ours$^{*}$
        & 52.6{\small \textcolor{gray}{$\pm$0.9}}
        & 56.6{\small \textcolor{gray}{$\pm$0.6}}
        & 61.7{\small \textcolor{gray}{$\pm$1.1}}
        & 64.3{\small \textcolor{gray}{$\pm$0.2}}
        & & 74.5{\small \textcolor{gray}{$\pm$0.5}}
        & 74.1{\small \textcolor{gray}{$\pm$0.2}}
        & 75.6{\small \textcolor{gray}{$\pm$0.4}}
        & 75.7{\small \textcolor{gray}{$\pm$1.0}}
        \\
        \midrule 
        Mirjalili~\etal~\cite{mirjalili2018semi}$^{**}$
        & \multicolumn{9}{l}{86.4 (G-COTS),\qquad 60.7 (IntraFace), \qquad 98.5 (Random Forest), \qquad 99.3 (Neural Network), \qquad 98.3 (SVM).}

        \\
        Dhar~\etal~\cite{dhar2020genderneutral}$^{***}$
        & \multicolumn{9}{l}{64.01 (Logistic Regression).}
        \\
        SensitiveNet~\cite{morales_2020_sensitivenets}$^{****}$
        & \multicolumn{9}{l}{65.2 (Random Forest), \qquad 65.7 (Neural Network), \qquad 67.3 (SVM).}
        \\
        \toprule
        \multicolumn{10}{l}{{\footnotesize $*$: For both the ArcFace only and ours entries: (1) all adversarial models are trained on full VGGFace2 training set ($\sim$3,140K images) and tested on our VGGFace2 test set}} \\
        \multicolumn{10}{l}{{\footnotesize \quad (25K images); (2) reported in balanced accuracy.}} \\
        \multicolumn{10}{l}{{\footnotesize $**$: (1) G-COTS and IntraFace are off-the-shelf softwares and are \emph{untrained}; their accuracies are from~\cite{mirjalili2018semi}. (2) Random Forest, Neural Network, and SVM are trained on CelebA}} \\
        \multicolumn{10}{l}{{\footnotesize \quad training set ($\sim$157K images); their accuracies are from~\cite{morales_2020_sensitivenets}. (3) The width, depth, and architecture of the Neutral Network are unknown. (4) All adversarial models are tested}} \\
        \multicolumn{10}{l}{{\footnotesize \quad  on CelebA test set ($\sim$19K images). (5) reported in accuracy.}} \\
        \multicolumn{10}{l}{{\footnotesize $**$$*$: (1) The adversarial model is trained on 60K IJB-C face images. (2) Test on 20K IJB-C face images. (3) The reported accuracy is from~\cite{dhar2020genderneutral}.}} \\
        \multicolumn{10}{l}{{\footnotesize $**$$**$: (1) The width, depth, and architecture of the Neutral Network are unknown. (2) All adversarial models are trained on CelebA training set ($\sim$157K images) and tested}} \\
        \multicolumn{10}{l}{{\footnotesize \quad on CelebA test set ($\sim$19K images). (3) The reported accuracies are from~\cite{morales_2020_sensitivenets}.}} \\
    \end{tabular}}
\end{table*}

When there is no correlation, our approach is able to learn nearly $0$-fair representation (\ie~the sensitive attribute's balanced accuracy is reduced to $\sim$50\%), with almost no cost in discriminativeness.
Using the learned representation, all three target attributes' predictions become much fairer for all considered fairness notions, in agreement with the theoretical fairness guarantees.

\vspace{1ex}
\noindent\textbf{Correlated MPI3D~}
In the real world, it is common that the sensitive attribute correlates with some of the target attributes.
To investigate the efficacy of our approach on such scenarios, we create data splits with controlled correlation.
Specifically, we let object shape correlate with the sensitive attribute object color.
Object size and background color remain independent.
Let
{\small
\begin{equation*}
\begin{cases}
    p_1 & \overset{\vartriangle}{=} P(\text{object shape}=\text{pyramid} \mid \text{object color}=\text{white}) ,\\
    p_2 & \overset{\vartriangle}{=} P(\text{object shape}=\text{pyramid} \mid \text{object color}=\text{red}) .
\end{cases}
\end{equation*}
}
Notably, $|p_1-p_2|$ is a measure of correlation between object shape and object color.
We create correlated data splits by varying $|p_1-p_2|\in \{0,0.2,0.4,0.6,0.8,1\}$ while maintaining $p_1+p_2=1$.

\begin{figure}[!t]
    \centering
    \includegraphics[width=0.75\linewidth]{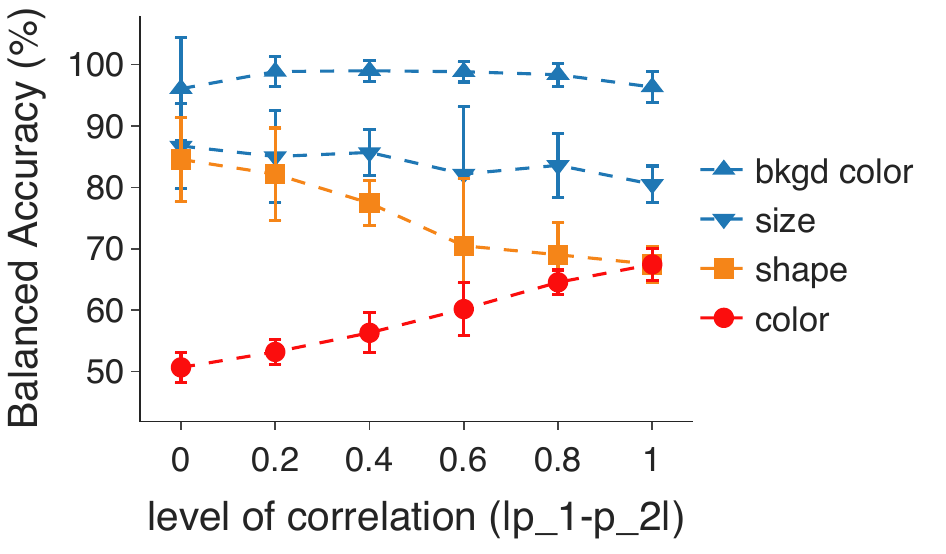}
    \caption{Our approach on correlated MPI3D dataset with varying level of correlation. Color is the sensitive attribute. Shape is the target attribute that correlates with color. Size and bkgd color are independent target attributes.}
    \label{fig:MPI3D_correlation}
    \vspace{-2ex}
\end{figure}

As shown in \figref{fig:MPI3D_correlation}, 
our method robustly learns fair and discriminative representation in the presence of correlation, although the sensitive attribute gradually becomes more predictable with increasing correlation.
The uncorrelated target attributes' prediction accuracy is only slightly affected.
The correlated target attribute's performance inevitably deteriorates because it increasingly correlates with the sensitive attribute.

\begin{table*}[t]
    \definecolor{green}{rgb}{0,0.45,0}
    \centering
    \caption{The Learned Face Representation Evaluated on CelebA Dataset Compared with Representations Learned with ArcFace Only (AO). We Report 10 Facial Attributes With Highest Balanced Accuracy (Complete Results Reported in the Supplementary Material). CLB is short for Calibration Lower Bound, which lower bounds DPC, DNC, and DC, see~\thref{thm:DPC}.}
    \label{tab:CelebA}
    \vspace{-1ex}
    \setlength\extrarowheight{0pt}
    \setlength{\tabcolsep}{2pt}
    \adjustbox{width=\textwidth}{
    \begin{tabular}{l l l l l l l l l l l l l l l l l l l l l l l l l c}
        \toprule
        \multirow{3}{*}{Attribute} &
        \multicolumn{2}{c}{Discriminativeness ($\uparrow$)} & &
        \multicolumn{20}{c}{Fairness ($\downarrow$)} & &
        CLB \\
        \cmidrule{2-3}
        \cmidrule{5-24}
        \cmidrule{26-26}
        & \multicolumn{2}{c}{BA} & & \multicolumn{2}{c}{SP} & & \multicolumn{2}{c}{DOpp} & & \multicolumn{2}{c}{DR} & & \multicolumn{2}{c}{DOdds} & & \multicolumn{2}{c}{DPC} & & \multicolumn{2}{c}{DNC} & & \multicolumn{2}{c}{DC} & & $\frac{1}{2}|a-b|$\\ 
        \midrule
        & AO & \cellcolor[HTML]{D0D0D0} Ours & & AO & \cellcolor[HTML]{D0D0D0} Ours & & AO & \cellcolor[HTML]{D0D0D0} Ours & & AO & \cellcolor[HTML]{D0D0D0} Ours & & AO & \cellcolor[HTML]{D0D0D0} Ours & & AO & \cellcolor[HTML]{D0D0D0} Ours & & AO & \cellcolor[HTML]{D0D0D0} Ours & & AO & \cellcolor[HTML]{D0D0D0} Ours &\\ 
        Bald
        & .953{\small \textcolor{gray}{$\pm$ .001 }} 
        & \cellcolor[HTML]{D0D0D0} .943{\small \textcolor{gray}{$\pm$.002}} 
        & & .114{\small \textcolor{gray}{$\pm$.001}} 
        & \cellcolor[HTML]{D0D0D0} .073{\small \textcolor{gray}{$\pm$.001}} 
        & & .650{\small \textcolor{gray}{$\pm$.073}} 
        & \cellcolor[HTML]{D0D0D0} .234{\small \textcolor{gray}{$\pm$.062}} 
        & & .072{\small \textcolor{gray}{$\pm$.001}} 
        & \cellcolor[HTML]{D0D0D0} .031{\small \textcolor{gray}{$\pm$.001}} 
        & & .361{\small \textcolor{gray}{$\pm$.037}} 
        & \cellcolor[HTML]{D0D0D0} .132{\small \textcolor{gray}{$\pm$.031}} 
        & & .027{\small \textcolor{gray}{$\pm$.0002}} 
        & \cellcolor[HTML]{D0D0D0} .027{\small \textcolor{gray}{$\pm$.0002}} 
        & & .228{\small \textcolor{gray}{$\pm$.003}} 
        & \cellcolor[HTML]{D0D0D0} .076{\small \textcolor{gray}{$\pm$.002}} 
        & & .127{\small \textcolor{gray}{$\pm$.001}} 
        & \cellcolor[HTML]{D0D0D0} .051{\small \textcolor{gray}{$\pm$.001}} 
        & & .027 \\
        
        Gray hair
        & .925{\small \textcolor{gray}{$\pm$ .002 }} 
        & \cellcolor[HTML]{D0D0D0} .926{\small \textcolor{gray}{$\pm$.001}} 
        & & .132{\small \textcolor{gray}{$\pm$.001}} 
        & \cellcolor[HTML]{D0D0D0} .075{\small \textcolor{gray}{$\pm$.0003}} 
        & & .243{\small \textcolor{gray}{$\pm$.013}} 
        & \cellcolor[HTML]{D0D0D0} .076{\small \textcolor{gray}{$\pm$.009}} 
        & & .072{\small \textcolor{gray}{$\pm$.001}} 
        & \cellcolor[HTML]{D0D0D0} .017{\small \textcolor{gray}{$\pm$.0004}} 
        & & .157{\small \textcolor{gray}{$\pm$.006}} 
        & \cellcolor[HTML]{D0D0D0} .046{\small \textcolor{gray}{$\pm$.005}} 
        & & .039{\small \textcolor{gray}{$\pm$.0003}} 
        & \cellcolor[HTML]{D0D0D0} .038{\small \textcolor{gray}{$\pm$.0004}} 
        & & .171{\small \textcolor{gray}{$\pm$.016}} 
        & \cellcolor[HTML]{D0D0D0} .059{\small \textcolor{gray}{$\pm$.0004}} 
        & & .105{\small \textcolor{gray}{$\pm$.008}} 
        & \cellcolor[HTML]{D0D0D0} .049{\small \textcolor{gray}{$\pm$.0003}} 
        & & .037 \\
        
        Sideburns
        & .921{\small \textcolor{gray}{$\pm$ .001 }} 
        & \cellcolor[HTML]{D0D0D0} .906{\small \textcolor{gray}{$\pm$.005}} 
        & & .295{\small \textcolor{gray}{$\pm$.004}} 
        & \cellcolor[HTML]{D0D0D0} .138{\small \textcolor{gray}{$\pm$.003}} 
        & & .261{\small \textcolor{gray}{$\pm$.096}} 
        & \cellcolor[HTML]{D0D0D0} .201{\small \textcolor{gray}{$\pm$.103}} 
        & & .214{\small \textcolor{gray}{$\pm$.004}} 
        & \cellcolor[HTML]{D0D0D0} .052{\small \textcolor{gray}{$\pm$.002}} 
        & & .237{\small \textcolor{gray}{$\pm$.047}} 
        & \cellcolor[HTML]{D0D0D0} .126{\small \textcolor{gray}{$\pm$.052}} 
        & & .067{\small \textcolor{gray}{$\pm$.0002}} 
        & \cellcolor[HTML]{D0D0D0} .067{\small \textcolor{gray}{$\pm$.0002}} 
        & & .639{\small \textcolor{gray}{$\pm$.013}} 
        & \cellcolor[HTML]{D0D0D0} .124{\small \textcolor{gray}{$\pm$.004}} 
        & & .353{\small \textcolor{gray}{$\pm$.006}} 
        & \cellcolor[HTML]{D0D0D0} .095{\small \textcolor{gray}{$\pm$.002}} 
        & & .068 \\
        
        Goatee
        & .919{\small \textcolor{gray}{$\pm$ .002 }} 
        & \cellcolor[HTML]{D0D0D0} .906{\small \textcolor{gray}{$\pm$.003}} 
        & & .330{\small \textcolor{gray}{$\pm$.008}} 
        & \cellcolor[HTML]{D0D0D0} .159{\small \textcolor{gray}{$\pm$.003}} 
        & & .314{\small \textcolor{gray}{$\pm$.033}} 
        & \cellcolor[HTML]{D0D0D0} .163{\small \textcolor{gray}{$\pm$.069}} 
        & & .243{\small \textcolor{gray}{$\pm$.008}} 
        & \cellcolor[HTML]{D0D0D0} .062{\small \textcolor{gray}{$\pm$.002}} 
        & & .278{\small \textcolor{gray}{$\pm$.018}} 
        & \cellcolor[HTML]{D0D0D0} .112{\small \textcolor{gray}{$\pm$.035}} 
        & & .075{\small \textcolor{gray}{$\pm$.0003}} 
        & \cellcolor[HTML]{D0D0D0} .075{\small \textcolor{gray}{$\pm$.0003}} 
        & & .731{\small \textcolor{gray}{$\pm$.021}} 
        & \cellcolor[HTML]{D0D0D0} .139{\small \textcolor{gray}{$\pm$.003}} 
        & & .403{\small \textcolor{gray}{$\pm$.010}} 
        & \cellcolor[HTML]{D0D0D0} .107{\small \textcolor{gray}{$\pm$.002}} 
        & & .075 \\
        
        Blond hair
        & .918{\small \textcolor{gray}{$\pm$ .001 }} 
        & \cellcolor[HTML]{D0D0D0} .912{\small \textcolor{gray}{$\pm$.002}} 
        & & .298{\small \textcolor{gray}{$\pm$.003}} 
        & \cellcolor[HTML]{D0D0D0} .188{\small \textcolor{gray}{$\pm$.004}} 
        & & .330{\small \textcolor{gray}{$\pm$.010}} 
        & \cellcolor[HTML]{D0D0D0} .176{\small \textcolor{gray}{$\pm$.008}} 
        & & .136{\small \textcolor{gray}{$\pm$.003}} 
        & \cellcolor[HTML]{D0D0D0} .021{\small \textcolor{gray}{$\pm$.004}} 
        & & .233{\small \textcolor{gray}{$\pm$.006}} 
        & \cellcolor[HTML]{D0D0D0} .098{\small \textcolor{gray}{$\pm$.005}} 
        & & .112{\small \textcolor{gray}{$\pm$.0002}} 
        & \cellcolor[HTML]{D0D0D0} .111{\small \textcolor{gray}{$\pm$.0002}} 
        & & .226{\small \textcolor{gray}{$\pm$.005}} 
        & \cellcolor[HTML]{D0D0D0} .128{\small \textcolor{gray}{$\pm$.004}} 
        & & .169{\small \textcolor{gray}{$\pm$.002}} 
        & \cellcolor[HTML]{D0D0D0} .119{\small \textcolor{gray}{$\pm$.002}} 
        & & .109 \\
        
        Eyeglasses
        & .916{\small \textcolor{gray}{$\pm$ .010 }} 
        & \cellcolor[HTML]{D0D0D0} .940{\small \textcolor{gray}{$\pm$.003}} 
        & & .133{\small \textcolor{gray}{$\pm$.009}} 
        & \cellcolor[HTML]{D0D0D0} .098{\small \textcolor{gray}{$\pm$.002}} 
        & & .240{\small \textcolor{gray}{$\pm$.029}} 
        & \cellcolor[HTML]{D0D0D0} .086{\small \textcolor{gray}{$\pm$.004}} 
        & & .051{\small \textcolor{gray}{$\pm$.010}} 
        & \cellcolor[HTML]{D0D0D0} .017{\small \textcolor{gray}{$\pm$.003}} 
        & & .146{\small \textcolor{gray}{$\pm$.017}} 
        & \cellcolor[HTML]{D0D0D0} .051{\small \textcolor{gray}{$\pm$.003}} 
        & & .051{\small \textcolor{gray}{$\pm$.001}} 
        & \cellcolor[HTML]{D0D0D0} .050{\small \textcolor{gray}{$\pm$.0002}} 
        & & .133{\small \textcolor{gray}{$\pm$.024}} 
        & \cellcolor[HTML]{D0D0D0} .071{\small \textcolor{gray}{$\pm$.007}} 
        & & .092{\small \textcolor{gray}{$\pm$.012}} 
        & \cellcolor[HTML]{D0D0D0} .061{\small \textcolor{gray}{$\pm$.004}} 
        & & .051 \\

        Mustache
        & .905{\small \textcolor{gray}{$\pm$ .007 }} 
        & \cellcolor[HTML]{D0D0D0} .901{\small \textcolor{gray}{$\pm$.004}} 
        & & .270{\small \textcolor{gray}{$\pm$.007}} 
        & \cellcolor[HTML]{D0D0D0} .115{\small \textcolor{gray}{$\pm$.002}} 
        & & .476{\small \textcolor{gray}{$\pm$.127}} 
        & \cellcolor[HTML]{D0D0D0} .483{\small \textcolor{gray}{$\pm$.093}} 
        & & .219{\small \textcolor{gray}{$\pm$.009}} 
        & \cellcolor[HTML]{D0D0D0} .053{\small \textcolor{gray}{$\pm$.001}} 
        & & .348{\small \textcolor{gray}{$\pm$.064}} 
        & \cellcolor[HTML]{D0D0D0} .268{\small \textcolor{gray}{$\pm$.047}} 
        & & .049{\small \textcolor{gray}{$\pm$.0001}} 
        & \cellcolor[HTML]{D0D0D0} .049{\small \textcolor{gray}{$\pm$.0001}} 
        & & .655{\small \textcolor{gray}{$\pm$.023}} 
        & \cellcolor[HTML]{D0D0D0} .122{\small \textcolor{gray}{$\pm$.003}} 
        & & .352{\small \textcolor{gray}{$\pm$.011}} 
        & \cellcolor[HTML]{D0D0D0} .085{\small \textcolor{gray}{$\pm$.001}} 
        & & .050 \\

        Wearing lipstick
        & .899{\small \textcolor{gray}{$\pm$ .001 }} 
        & \cellcolor[HTML]{D0D0D0} .788{\small \textcolor{gray}{$\pm$.003}} 
        & & .889{\small \textcolor{gray}{$\pm$.002}} 
        & \cellcolor[HTML]{D0D0D0} .511{\small \textcolor{gray}{$\pm$.006}} 
        & & .587{\small \textcolor{gray}{$\pm$.019}} 
        & \cellcolor[HTML]{D0D0D0} .249{\small \textcolor{gray}{$\pm$.016}} 
        & & .744{\small \textcolor{gray}{$\pm$.005}} 
        & \cellcolor[HTML]{D0D0D0} .286{\small \textcolor{gray}{$\pm$.006}} 
        & & .665{\small \textcolor{gray}{$\pm$.011}} 
        & \cellcolor[HTML]{D0D0D0} .268{\small \textcolor{gray}{$\pm$.011}} 
        & & .402{\small \textcolor{gray}{$\pm$.001}} 
        & \cellcolor[HTML]{D0D0D0} .399{\small \textcolor{gray}{$\pm$.0004}} 
        & & .542{\small \textcolor{gray}{$\pm$.001}} 
        & \cellcolor[HTML]{D0D0D0} .418{\small \textcolor{gray}{$\pm$.002}} 
        & & .472{\small \textcolor{gray}{$\pm$.001}} 
        & \cellcolor[HTML]{D0D0D0} .408{\small \textcolor{gray}{$\pm$.001}} 
        & & .400 \\

        5 o'clock shadow
        & .883{\small \textcolor{gray}{$\pm$ .001 }} 
        & \cellcolor[HTML]{D0D0D0} .867{\small \textcolor{gray}{$\pm$.005}} 
        & & .574{\small \textcolor{gray}{$\pm$.008}} 
        & \cellcolor[HTML]{D0D0D0} .289{\small \textcolor{gray}{$\pm$.006}} 
        & & .205{\small \textcolor{gray}{$\pm$.030}} 
        & \cellcolor[HTML]{D0D0D0} .142{\small \textcolor{gray}{$\pm$.033}} 
        & & .482{\small \textcolor{gray}{$\pm$.010}} 
        & \cellcolor[HTML]{D0D0D0} .166{\small \textcolor{gray}{$\pm$.004}} 
        & & .344{\small \textcolor{gray}{$\pm$.014}} 
        & \cellcolor[HTML]{D0D0D0} .154{\small \textcolor{gray}{$\pm$.016}} 
        & & .133{\small \textcolor{gray}{$\pm$.001}} 
        & \cellcolor[HTML]{D0D0D0} .133{\small \textcolor{gray}{$\pm$.001}} 
        & & .815{\small \textcolor{gray}{$\pm$.002}} 
        & \cellcolor[HTML]{D0D0D0} .265{\small \textcolor{gray}{$\pm$.007}} 
        & & .474{\small \textcolor{gray}{$\pm$.001}} 
        & \cellcolor[HTML]{D0D0D0} .199{\small \textcolor{gray}{$\pm$.003}} 
        & & .133 \\

        Wearing hat
        & .875{\small \textcolor{gray}{$\pm$ .010 }} 
        & \cellcolor[HTML]{D0D0D0} .893{\small \textcolor{gray}{$\pm$.006}} 
        & & .083{\small \textcolor{gray}{$\pm$.009}} 
        & \cellcolor[HTML]{D0D0D0} .055{\small \textcolor{gray}{$\pm$.003}} 
        & & .221{\small \textcolor{gray}{$\pm$.014}} 
        & \cellcolor[HTML]{D0D0D0} .080{\small \textcolor{gray}{$\pm$.010}} 
        & & .038{\small \textcolor{gray}{$\pm$.010}} 
        & \cellcolor[HTML]{D0D0D0} .014{\small \textcolor{gray}{$\pm$.004}} 
        & & .129{\small \textcolor{gray}{$\pm$.008}} 
        & \cellcolor[HTML]{D0D0D0} .047{\small \textcolor{gray}{$\pm$.007}} 
        & & .030{\small \textcolor{gray}{$\pm$.0003}} 
        & \cellcolor[HTML]{D0D0D0} .029{\small \textcolor{gray}{$\pm$.0002}} 
        & & .172{\small \textcolor{gray}{$\pm$.036}} 
        & \cellcolor[HTML]{D0D0D0} .053{\small \textcolor{gray}{$\pm$.011}} 
        & & .101{\small \textcolor{gray}{$\pm$.018}} 
        & \cellcolor[HTML]{D0D0D0} .041{\small \textcolor{gray}{$\pm$.005}} 
        & & .028 \\

        \bottomrule

    \end{tabular}}
\end{table*}
\subsection{Experiment on Face Datasets}

\vspace{1ex}
\noindent\textbf{Experimental Settings~}
Gender is the sensitive attribute.
We use SphereNet-20~\cite{liu2017sphereface} as the neural architecture and additive Angular Margin Loss (ArcFace)~\cite{deng_CVPR_2019_arcface} as the pretext loss.
The representation dimension is 32 and the final $\lambda$ (after line search) is 30.
We train the model on VGGFace2 training set for 200k iterations with Adam optimizer~\cite{kingma_2014_adam} and batch size 256.
The learning rate starts from 1e-4 and is divided by 10 at 140k and 180k iterations.

We first evaluate the learned representation on face verification, face identification, and gender prediction.

For \textbf{face verification}, the objective is to predict whether a given pair of face images is of the same identity.
We evaluate on the LFW test set, which contains 10 data splits, each with 300 positive and 300 negative pairs.
Cosine similarity is computed and 10-fold cross validation is used to select the threshold.
We report the True Accept Rates (TAR) at 1e-5, 1e-4, 1e-3, and 1e-2 False Accept Rate (FAR).

\textbf{Face Identification} assumes a probe image and a gallery set that contains only one target image of the same identity and many other distractors.
The task is, given a probe image, to find the target image from the gallery set.
Our VGGFace2 test set consists of 50 images per identity for 500 identities.
Each image is posed as the probe 49 times.
Each time one of the other 49 images of same identity is posed as the target.
All other images of different identities are used as distractors.
This results in total 1.25 million tests.
Cosine similarity is used as the similarity measure.
We report the top-K accuracy with $K\in\{1,10,100,1000\}$.

For \textbf{gender prediction}, we pose various adversarial classifiers to predict gender using the learned representation.
They are trained on the full VGGFace2 training set. 
We report the Balanced Accuracy on our VGGFace2 test set.

\vspace{1ex}
\noindent\textbf{Results on Face Verification, Identification, and Gender Prediction~}
\tabref{tab:face} reports the experimental results.
Our approach learns gender-blind face representations, from which various classifiers cannot infer gender with high accuracy, with minor cost in the performances of face verification and identification.

Although we cannot directly compare the gender unpredictability with existing methods due to difference in evaluation (see footnote of \tabref{tab:face}), it is clear that our approach is competitive.
First, our approach readily transforms a pretext loss---ArcFace here---for gender-blind face representation learning and is much simpler compared to specifically designed methods~\cite{mirjalili2018semi,terhorst2019unsupervised,dhar2020genderneutral,morales_2020_sensitivenets}.
Second, we evaluate against a wide range of classifiers and allow them to be trained on a much larger training set ($52\times$ larger compared to~\cite{dhar2020genderneutral} and $20\times$ larger compared to~\cite{morales_2020_sensitivenets}).
Third, we also test on a larger test set (25K compared to 20K in~\cite{dhar2020genderneutral} and 19K in~\cite{morales_2020_sensitivenets}).
Even under significantly stricter evaluation, the achieved gender unpredictability is still on par with existing work.
This demonstrates the effectiveness of our proposed approach for learning gender-blind face representations.

\vspace{1ex}
\noindent\textbf{Results on Facial Attribute Prediction~}
Since both LFW and VGGFace2 datasets are without facial attributes, we evaluate the learned face representation for facial attribute prediction on CelebA dataset instead.
This results in a generalization evaluation because the training and the test datasets are no longer from the same distribution.

We report the results in~\tabref{tab:CelebA}.
We report 10 facial attributes with highest prediction accuracy for brevity, with complete results reported in the supplementary material.
Interestingly, predictions using naively learned face representations are often calibrated w.r.t. DPC (achieve the lower bound) but uncalibrated w.r.t. DNC.
Our approach maintains the same level of balanced accuracy for facial attribute prediction and greatly improves their fairness for all seven group fairness notions.
The achieved values for DPC, DNC, and DC are all close to their lower bounds.
As the complete result in the supplementary material demonstrates, fairness is improved even for the predictions that are less accurate.

We note that we are unaware of these facial attributes when learning face representation.
Thus, \tabref{tab:CelebA} empirically demonstrates that using a both fair and discriminative representation, we can indeed approximately achieve seven group fairness notions \emph{for downstream unknown prediction tasks}.


\section{Conclusion} \label{sec:conclusion}
In this work, we prove that fair representation guarantees approximate seven group fairness notions for all prediction tasks for which the representation is discriminative.
With a sharp characterization, we provide a better understanding of what fair representation can and cannot guarantee.
We have considered a setting where the dataset users can be adversarial.
An intriguing open question is, what fair representation can guarantee if we restrict to accuracy-driven dataset users.
Such knowledge allows the data owner to guarantee better fairness when they are able to verify the accuracy of downstream predictions.

\section*{ACKNOWLEDGMENTS}

This research/project is supported by the National Research Foundation, Singapore under its Strategic Capability Research Centres Funding Initiative. 
Any opinions, findings and conclusions or recommendations expressed in this material are those of the author(s) and do not reflect the views of National Research Foundation, Singapore.

\bibliographystyle{ieeetr}
\bibliography{reference}

\begin{thebibliography}{10}

\bibitem{obermeyer_Science_2019}
Z.~Obermeyer, B.~Powers, C.~Vogeli, and S.~Mullainathan, ``Dissecting racial
  bias in an algorithm used to manage the health of populations,'' {\em
  Science}, vol.~366, no.~6464, pp.~447--453, 2019.

\bibitem{buolamwini_FAccT_2018_gender}
J.~Buolamwini and T.~Gebru, ``Gender shades: Intersectional accuracy
  disparities in commercial gender classification,'' in {\em Conference on
  Fairness, Accountability and Transparency}, vol.~81 of {\em Proceedings of
  Machine Learning Research}, pp.~77--91, 2018.

\bibitem{angwin_2016_propublica}
J.~Angwin, J.~Larson, S.~Mattu, and L.~Kirchner, ``Machine bias,'' {\em
  ProPublica, May}, vol.~23, p.~2016, 2016.

\bibitem{berg2018rise}
T.~Berg, V.~Burg, A.~Gombovi{\'e}, and M.~Puri, ``On the rise of the
  fintechs---credit scoring using digital footprints. federal deposit insurance
  corporation,'' {\em Center for Financial Research WP}, vol.~4, 2018.

\bibitem{jernigan2009gaydar}
C.~Jernigan and B.~F. Mistree, ``Gaydar: Facebook friendships expose sexual
  orientation,'' {\em First Monday}, 2009.

\bibitem{kosinski2013private}
M.~Kosinski, D.~Stillwell, and T.~Graepel, ``Private traits and attributes are
  predictable from digital records of human behavior,'' {\em Proceedings of the
  national academy of sciences}, vol.~110, no.~15, pp.~5802--5805, 2013.

\bibitem{preoctiuc2017beyond}
D.~Preo{\c{t}}iuc-Pietro, Y.~Liu, D.~Hopkins, and L.~Ungar, ``Beyond binary
  labels: political ideology prediction of twitter users,'' in {\em Proceedings
  of the 55th Annual Meeting of the Association for Computational Linguistics
  (Volume 1: Long Papers)}, pp.~729--740, 2017.

\bibitem{beutel_arxiv_2017}
A.~Beutel, J.~Chen, Z.~Zhao, and E.~H. Chi, ``Data decisions and theoretical
  implications when adversarially learning fair representations,'' in {\em
  SIGKDD Conference on Knowledge Discovery and Data Mining, Workshop on
  Fairness, Accountability, and Transparency in Machine Learning}, 2017.

\bibitem{edwards_arxiv_2015}
H.~Edwards and A.~Storkey, ``Censoring representations with an adversary,'' in
  {\em International Conference on Learning Representations}, 2015.

\bibitem{xie_NeurIPS_2017}
Q.~Xie, Z.~Dai, Y.~Du, E.~Hovy, and G.~Neubig, ``Controllable invariance
  through adversarial feature learning,'' in {\em Advances in Neural
  Information Processing Systems}, pp.~585--596, 2017.

\bibitem{zhao_ICLR_2020_conditional}
H.~Zhao, A.~Coston, T.~Adel, and G.~J. Gordon, ``Conditional learning of fair
  representations,'' in {\em International Conference on Learning
  Representations}, 2020.

\bibitem{kingma_arxiv_2013_VAE}
D.~P. Kingma and M.~Welling, ``Auto-encoding variational bayes,'' in {\em
  International Conference on Learning Representations.}, 2013.

\bibitem{moyer_NeurIPS_2018}
D.~Moyer, S.~Gao, R.~Brekelmans, A.~Galstyan, and G.~Ver~Steeg, ``Invariant
  representations without adversarial training,'' in {\em Advances in Neural
  Information Processing Systems}, pp.~9084--9093, 2018.

\bibitem{song_AISTATS_2019}
J.~Song, P.~Kalluri, A.~Grover, S.~Zhao, and S.~Ermon, ``Learning controllable
  fair representations,'' in {\em International Conference on Artificial
  Intelligence and Statistics}, pp.~2164--2173, 2019.

\bibitem{louizos_arxiv_2015}
C.~Louizos, K.~Swersky, Y.~Li, M.~Welling, and R.~Zemel, ``The variational fair
  autoencoder,'' in {\em International Conference on Learning Representations},
  2016.

\bibitem{creager_ICML_2019}
E.~Creager, D.~Madras, J.-H. Jacobsen, M.~Weis, K.~Swersky, T.~Pitassi, and
  R.~Zemel, ``Flexibly fair representation learning by disentanglement,'' in
  {\em International Conference on Machine Learning}, vol.~97 of {\em
  Proceedings of Machine Learning Research}, pp.~1436--1445, 2019.

\bibitem{kehrenberg2020null}
T.~Kehrenberg, M.~Bartlett, O.~Thomas, and N.~Quadrianto, ``Null-sampling for
  interpretable and fair representations,'' in {\em European Conference on
  Computer Vision}, pp.~565--580, Springer, 2020.

\bibitem{quadrianto2019discovering}
N.~Quadrianto, V.~Sharmanska, and O.~Thomas, ``Discovering fair representations
  in the data domain,'' in {\em Proceedings of the IEEE/CVF Conference on
  Computer Vision and Pattern Recognition}, pp.~8227--8236, 2019.

\bibitem{karkkainen2021fairface}
K.~Karkkainen and J.~Joo, ``Fairface: Face attribute dataset for balanced race,
  gender, and age for bias measurement and mitigation,'' in {\em Proceedings of
  the IEEE/CVF Winter Conference on Applications of Computer Vision},
  pp.~1548--1558, 2021.

\bibitem{choi2020fair}
K.~Choi, A.~Grover, T.~Singh, R.~Shu, and S.~Ermon, ``Fair generative modeling
  via weak supervision,'' in {\em International Conference on Machine
  Learning}, vol.~119, pp.~1887--1898, PMLR, 2020.

\bibitem{ramaswamy2021fair}
V.~V. Ramaswamy, S.~S. Kim, and O.~Russakovsky, ``Fair attribute classification
  through latent space de-biasing,'' in {\em Proceedings of the IEEE/CVF
  Conference on Computer Vision and Pattern Recognition}, pp.~9301--9310, 2021.

\bibitem{bolukbasi_NeurIPS_2016_word_embeddings}
T.~Bolukbasi, K.-W. Chang, J.~Y. Zou, V.~Saligrama, and A.~T. Kalai, ``Man is
  to computer programmer as woman is to homemaker? debiasing word embeddings,''
  in {\em Advances in Neural Information Processing Systems}, pp.~4349--4357,
  2016.

\bibitem{caliskan2017semantics}
A.~Caliskan, J.~J. Bryson, and A.~Narayanan, ``Semantics derived automatically
  from language corpora contain human-like biases,'' {\em Science}, vol.~356,
  no.~6334, pp.~183--186, 2017.

\bibitem{zhao2018learning}
J.~Zhao, Y.~Zhou, Z.~Li, W.~Wang, and K.-W. Chang, ``Learning gender-neutral
  word embeddings,'' in {\em Proceedings of the 2018 Conference on Empirical
  Methods in Natural Language Processing}, pp.~4847--4853, 2018.

\bibitem{may2019measuring}
C.~May, A.~Wang, S.~Bordia, S.~Bowman, and R.~Rudinger, ``On measuring social
  biases in sentence encoders,'' in {\em Proceedings of the 2019 Conference of
  the North American Chapter of the Association for Computational Linguistics:
  Human Language Technologies, Volume 1 (Long and Short Papers)}, pp.~622--628,
  2019.

\bibitem{gardner2020evaluating}
M.~Gardner, Y.~Artzi, V.~Basmov, J.~Berant, B.~Bogin, S.~Chen, P.~Dasigi,
  D.~Dua, Y.~Elazar, A.~Gottumukkala, {\em et~al.}, ``Evaluating models’
  local decision boundaries via contrast sets,'' in {\em Proceedings of the
  2020 Conference on Empirical Methods in Natural Language Processing:
  Findings}, pp.~1307--1323, 2020.

\bibitem{feldman_SIGKDD_2015}
M.~Feldman, S.~A. Friedler, J.~Moeller, C.~Scheidegger, and
  S.~Venkatasubramanian, ``Certifying and removing disparate impact,'' in {\em
  ACM SIGKDD International Conference on Knowledge Discovery and Data Mining},
  pp.~259--268, 2015.

\bibitem{mcnamara_arxiv_2017}
D.~McNamara, C.~S. Ong, and R.~C. Williamson, ``Provably fair
  representations,'' {\em arXiv preprint arXiv:1710.04394}, 2017.

\bibitem{madras_ICML_2018}
D.~Madras, E.~Creager, T.~Pitassi, and R.~Zemel, ``Learning adversarially fair
  and transferable representations,'' in {\em International Conference on
  Machine Learning}, vol.~80, pp.~3384--3393, 2018.

\bibitem{calders_DMKD_2010}
T.~Calders and S.~Verwer, ``Three naive bayes approaches for
  discrimination-free classification,'' {\em Data Mining and Knowledge
  Discovery}, vol.~21, no.~2, pp.~277--292, 2010.

\bibitem{hardt_NeurIPS_2016_equality}
M.~Hardt, E.~Price, and N.~Srebro, ``Equality of opportunity in supervised
  learning,'' in {\em Advances in Neural Information Processing Systems},
  pp.~3315--3323, 2016.

\bibitem{narayanan2018translation}
A.~Narayanan, ``Translation tutorial: 21 fairness definitions and their
  politics,'' in {\em Conference on Fairness, Accountability, and
  Transparency}, vol.~1170, 2018.

\bibitem{chouldechova2017fair}
A.~Chouldechova, ``Fair prediction with disparate impact: A study of bias in
  recidivism prediction instruments,'' {\em Big data}, vol.~5, no.~2,
  pp.~153--163, 2017.

\bibitem{kleinberg2018inherent}
J.~Kleinberg, ``Inherent trade-offs in algorithmic fairness,'' in {\em
  Abstracts of the 2018 ACM International Conference on Measurement and
  Modeling of Computer Systems}, pp.~40--40, 2018.

\bibitem{barocas-hardt-narayanan}
S.~Barocas, M.~Hardt, and A.~Narayanan, {\em Fairness and Machine Learning}.
\newblock fairmlbook.org, 2019.
\newblock \url{http://www.fairmlbook.org}.

\bibitem{oneto_2020_MMD}
L.~Oneto, M.~Donini, G.~Luise, C.~Ciliberto, A.~Maurer, and M.~Pontil,
  ``Exploiting mmd and sinkhorn divergences for fair and transferable
  representation learning,'' {\em Advances in Neural Information Processing
  Systems}, vol.~33, 2020.

\bibitem{gretton_JMLR_2012_mmd}
A.~Gretton, K.~M. Borgwardt, M.~J. Rasch, B.~Sch{\"o}lkopf, and A.~Smola, ``A
  kernel two-sample test,'' {\em The Journal of Machine Learning Research},
  vol.~13, no.~1, pp.~723--773, 2012.

\bibitem{zhao_NeurIPS_2019}
H.~Zhao and G.~Gordon, ``Inherent tradeoffs in learning fair representations,''
  in {\em Advances in Neural Information Processing Systems}, pp.~15675--15685,
  2019.

\bibitem{kim_AAAI_2020}
J.-Y. Kim and S.-B. Cho, ``Fair representation for safe artificial intelligence
  via adversarial learning of unbiased information bottleneck,'' in {\em
  SafeAI@ AAAI}, pp.~105--112, 2020.

\bibitem{zhao_NC_2015_fastmmd}
J.~Zhao and D.~Meng, ``Fastmmd: Ensemble of circular discrepancy for efficient
  two-sample test.,'' {\em Neural Computation}, vol.~27, no.~6, pp.~1345--1372,
  2015.

\bibitem{menon_FAcct_2018}
A.~K. Menon and R.~C. Williamson, ``The cost of fairness in binary
  classification,'' in {\em Conference on Fairness, Accountability and
  Transparency}, vol.~81, pp.~107--118, PMLR, 2018.

\bibitem{dutta2020there}
S.~Dutta, D.~Wei, H.~Yueksel, P.-Y. Chen, S.~Liu, and K.~Varshney, ``Is there a
  trade-off between fairness and accuracy? a perspective using mismatched
  hypothesis testing,'' in {\em International Conference on Machine Learning},
  vol.~119, pp.~2803--2813, PMLR, 2020.

\bibitem{lechner2021impossibility}
T.~Lechner, S.~Ben-David, S.~Agarwal, and N.~Ananthakrishnan, ``Impossibility
  results for fair representations,'' {\em arXiv preprint arXiv:2107.03483},
  2021.

\bibitem{levin_2017_markov_chains}
D.~A. Levin and Y.~Peres, {\em Markov chains and mixing times}, vol.~107.
\newblock American Mathematical Soc., 2017.

\bibitem{liu2019implicit}
L.~T. Liu, M.~Simchowitz, and M.~Hardt, ``The implicit fairness criterion of
  unconstrained learning,'' in {\em International Conference on Machine
  Learning}, vol.~97, pp.~4051--4060, PMLR, 2019.

\bibitem{dantzig1998linear}
G.~B. Dantzig, {\em Linear programming and extensions}, vol.~48.
\newblock Princeton university press, 1998.

\bibitem{andersen2000mosek}
E.~D. Andersen and K.~D. Andersen, ``The mosek interior point optimizer for
  linear programming: an implementation of the homogeneous algorithm,'' in {\em
  High performance optimization}, pp.~197--232, Springer, 2000.

\bibitem{deng_CVPR_2019_arcface}
J.~Deng, J.~Guo, N.~Xue, and S.~Zafeiriou, ``Arcface: Additive angular margin
  loss for deep face recognition,'' in {\em Proceedings of the IEEE Conference
  on Computer Vision and Pattern Recognition}, pp.~4690--4699, 2019.

\bibitem{muller_1997_IPMs}
A.~M{\"u}ller, ``Integral probability metrics and their generating classes of
  functions,'' {\em Advances in Applied Probability}, pp.~429--443, 1997.

\bibitem{steinwart2001influence}
I.~Steinwart, ``On the influence of the kernel on the consistency of support
  vector machines,'' {\em Journal of machine learning research}, vol.~2,
  no.~Nov, pp.~67--93, 2001.

\bibitem{sriperumbudur_EJS_2012}
B.~K. Sriperumbudur, K.~Fukumizu, A.~Gretton, B.~Sch{\"o}lkopf, G.~R.
  Lanckriet, {\em et~al.}, ``On the empirical estimation of integral
  probability metrics,'' {\em Electronic Journal of Statistics}, vol.~6,
  pp.~1550--1599, 2012.

\bibitem{ramdas_AAAI_2015}
A.~Ramdas, S.~J. Reddi, B.~P\'{o}czos, A.~Singh, and L.~Wasserman, ``On the
  decreasing power of kernel and distance based nonparametric hypothesis tests
  in high dimensions,'' in {\em Proceedings of the Twenty-Ninth AAAI Conference
  on Artificial Intelligence}, pp.~3571--3577, 2015.

\bibitem{rasmussen_2003_GPML}
C.~E. Rasmussen, ``Gaussian processes in machine learning,'' in {\em Summer
  School on Machine Learning}, pp.~63--71, Springer, 2003.

\bibitem{freund1999adaptive}
Y.~Freund and R.~E. Schapire, ``Adaptive game playing using multiplicative
  weights,'' {\em Games and Economic Behavior}, vol.~29, no.~1-2, pp.~79--103,
  1999.

\bibitem{zinkevich2003online}
M.~Zinkevich, ``Online convex programming and generalized infinitesimal
  gradient ascent,'' in {\em Proceedings of the 20th international conference
  on machine learning}, pp.~928--936, 2003.

\bibitem{golowich2020last}
N.~Golowich, S.~Pattathil, C.~Daskalakis, and A.~Ozdaglar, ``Last iterate is
  slower than averaged iterate in smooth convex-concave saddle point
  problems,'' in {\em Conference on Learning Theory}, pp.~1758--1784, PMLR,
  2020.

\bibitem{zemel_ICML_2013_LFR}
R.~Zemel, Y.~Wu, K.~Swersky, T.~Pitassi, and C.~Dwork, ``Learning fair
  representations,'' in {\em International Conference on Machine Learning},
  pp.~325--333, 2013.

\bibitem{kohavi_KDD_1996}
R.~Kohavi, ``Scaling up the accuracy of naive-bayes classifiers: a
  decision-tree hybrid,'' in {\em Proceedings of the Second International
  Conference on Knowledge Discovery and Data Mining}, pp.~202--207, 1996.

\bibitem{gondal_NeurIPS_2019_MPI3D}
M.~W. Gondal, M.~Wuthrich, D.~Miladinovic, F.~Locatello, M.~Breidt,
  V.~Volchkov, J.~Akpo, O.~Bachem, B.~Sch{\"o}lkopf, and S.~Bauer, ``On the
  transfer of inductive bias from simulation to the real world: a new
  disentanglement dataset,'' in {\em Advances in Neural Information Processing
  Systems}, pp.~15740--15751, 2019.

\bibitem{cao_FG_2018_VGGFace2}
Q.~Cao, L.~Shen, W.~Xie, O.~M. Parkhi, and A.~Zisserman, ``Vggface2: A dataset
  for recognising faces across pose and age,'' in {\em International Conference
  on Automatic Face and Gesture Recognition}, pp.~67--74, 2018.

\bibitem{liu2017sphereface}
W.~Liu, Y.~Wen, Z.~Yu, M.~Li, B.~Raj, and L.~Song, ``Sphereface: Deep
  hypersphere embedding for face recognition,'' in {\em Proceedings of the IEEE
  Conference on Computer Vision and Pattern Recognition}, pp.~212--220, 2017.

\bibitem{huang_2008_LFW}
G.~B. Huang, M.~Ramesh, T.~Berg, and E.~Learned-Miller, ``Labeled faces in the
  wild: A database for studying face recognition in unconstrained
  environments,'' Tech. Rep. 07-49, University of Massachusetts, Amherst,
  October 2007.

\bibitem{liu2015deep}
Z.~Liu, P.~Luo, X.~Wang, and X.~Tang, ``Deep learning face attributes in the
  wild,'' in {\em Proceedings of the IEEE International Conference on Computer
  Vision}, pp.~3730--3738, 2015.

\bibitem{equal_1978_department}
E.~E.~O. Commission, C.~S. Commission, {\em et~al.}, ``Department of labor, \&
  department of justice.(1978). uniform guidelines on employee selection
  procedures,'' {\em Federal Register}, vol.~43, no.~166, pp.~38290--38315,
  1978.

\bibitem{he_CVPR_2016_resnet}
K.~He, X.~Zhang, S.~Ren, and J.~Sun, ``Deep residual learning for image
  recognition,'' in {\em Proceedings of the IEEE Conference on Computer Vision
  and Pattern Recognition}, pp.~770--778, 2016.

\bibitem{mirjalili2018semi}
V.~Mirjalili, S.~Raschka, A.~Namboodiri, and A.~Ross, ``Semi-adversarial
  networks: Convolutional autoencoders for imparting privacy to face images,''
  in {\em 2018 International Conference on Biometrics}, pp.~82--89, IEEE, 2018.

\bibitem{dhar2020genderneutral}
P.~Dhar, J.~Gleason, H.~Souri, C.~D. Castillo, and R.~Chellappa, ``Towards
  gender-neutral face descriptors for mitigating bias in face recognition,''
  {\em arXiv preprint arXiv:2006.07845}, 2020.

\bibitem{morales_2020_sensitivenets}
A.~Morales, J.~Fierrez, R.~Vera-Rodriguez, and R.~Tolosana, ``{SensitiveNets}:
  Learning agnostic representations with application to face images,'' {\em
  IEEE Transactions on Pattern Analysis and Machine Intelligence}, vol.~43,
  no.~06, pp.~2158--2164, 2021.

\bibitem{kingma_2014_adam}
D.~P. Kingma and J.~Ba, ``Adam: {A} method for stochastic optimization,'' in
  {\em International Conference on Learning Representations}, 2015.

\bibitem{terhorst2019unsupervised}
P.~Terh{\"o}rst, N.~Damer, F.~Kirchbuchner, and A.~Kuijper, ``Unsupervised
  privacy-enhancement of face representations using similarity-sensitive noise
  transformations,'' {\em Applied Intelligence}, vol.~49, no.~8,
  pp.~3043--3060, 2019.

\end{thebibliography}

\begin{IEEEbiography}[{\includegraphics[width=1in,height=1.25in,clip,keepaspectratio]{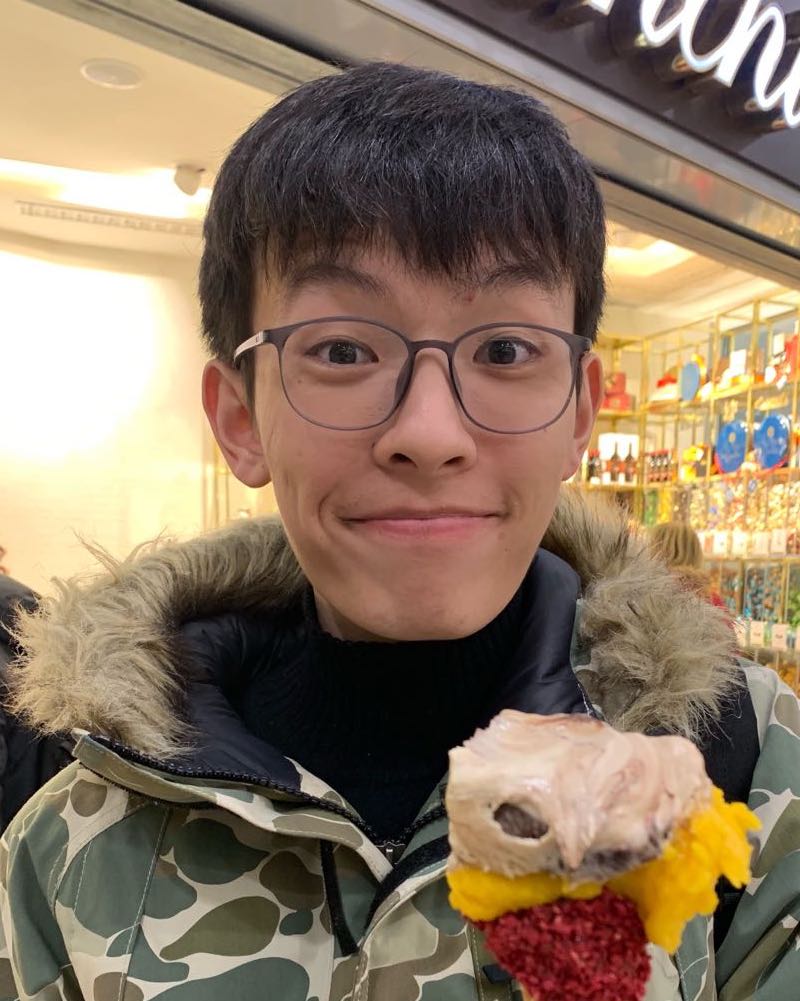}}]{Xudong~Shen}
    is a PhD candidate at the Integrative Sciences and Engineering Programme, Graduate School, National University of Singapore.
    He obtained his BEng from Zhejiang University, China, in 2019.
    His current research interest is fairness in machine learning.
\end{IEEEbiography}

\begin{IEEEbiography}[{\includegraphics[width=1in,height=1.25in,clip,keepaspectratio]{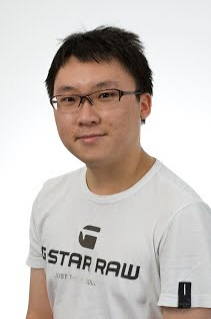}}]{Yongkang~Wong}
	is a senior research fellow at the School of Computing, National University of Singapore. 
	He is also the Assistant Director of the NUS Centre for Research in Privacy Technologies (N-CRiPT). 
	He obtained his BEng from the University of Adelaide and PhD from the University of Queensland. 
	He has worked as a graduate researcher at NICTA's Queensland laboratory, Brisbane, OLD, Australia, from 2008 to 2012. 
	His current research interests are in the areas of Image/Video Processing, Machine Learning, and Human Centric Analysis. 
	He is a member of the IEEE since 2009.
\end{IEEEbiography}

\begin{IEEEbiography}[{\includegraphics[width=1in,height=1.25in,clip,keepaspectratio]{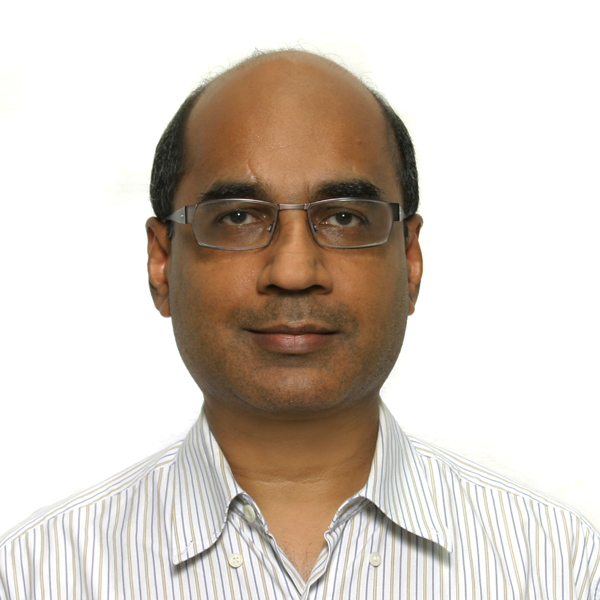}}]{Mohan~S.~Kankanhalli}
is Provost's Chair Professor of Computer Science at the National University of Singapore (NUS). He is the Dean of NUS School of Computing and he also directs N-CRiPT (NUS Centre for Research in Privacy Technologies) which conducts research on privacy on structured as well as unstructured (multimedia, sensors, IoT) data. Mohan obtained his BTech from IIT Kharagpur and MS \& PhD from the Rensselaer Polytechnic Institute. Mohan’s research interests are in Multimedia Computing, Computer Vision, Information Security \& Privacy, and Image/Video Processing. He has made many contributions in the area of multimedia \& vision---image and video understanding, data fusion, visual saliency as well as multimedia security---content authentication and privacy, and multi-camera surveillance. Mohan is a Fellow of IEEE.
\end{IEEEbiography}



\end{document}